%% file: multiMAB.tex
\documentclass[twoside,11pt]{article}
\usepackage{amsmath}
\usepackage{amssymb}
\usepackage{amsthm} 

\usepackage{jmlr2e}
\usepackage{mysetup-JMLR,multiMAB_setup}
\usepackage{graphicx}
\usepackage{natbib}
\usepackage{epstopdf}
\usepackage{algorithm}
\usepackage{subfigure}
\usepackage[noend]{algorithmic}
\usepackage{soul,color} 
\usepackage{ifthen}

\renewcommand{\eqref}[1]{Equation~(\ref{#1})}

\newcounter{RevisionNumber} 
\newcommand{\hlc}[2]{\ifthenelse{\value{RevisionNumber}=#1}{\hl{#2}}{#2}}
\renewcommand{\ASnote}[2]{\ifthenelse{\value{RevisionNumber}=#1}{\hl{#2}}{}}

\newcommand{\hlFR}[1]{{\sethlcolor{green}\hl{#1}\sethlcolor{yellow}}}

\newcommand{\TAB}{\hspace{6mm}} 

\renewcommand{\hl}[1]{#1}

\ShortHeadings{Ranked bandits in metric spaces}{Aleksandrs Slivkins, Filip Radlinski and Sreenivas Gollapudi}
\firstpageno{1}

\renewcommand{\D}{\ensuremath{D}}

\begin{document}
\title{Ranked bandits in metric spaces:\\
learning diverse rankings over large document collections\footnotemark[1]}
\author{\name Aleksandrs Slivkins \email slivkins@microsoft.com \\
        \addr Microsoft Research Silicon Valley \\ 1065 La Avenida, Mountain View, CA 94043, USA
        \AND
        \name Filip Radlinski \email filiprad@microsoft.com \\
        \addr Microsoft Research Cambridge\\ 7 J.J. Thomson Ave., Cambridge UK
        \AND
        \name Sreenivas Gollapudi \email sreenig@microsoft.com \\
        \addr Microsoft Research Silicon Valley \\ 1065 La Avenida, Mountain View, CA 94043, USA
}
\editor{}

\maketitle

\begin{abstract}
Most learning to rank research has assumed that the utility of different documents is independent, which results in learned ranking functions that return redundant results. The few approaches that avoid this have rather unsatisfyingly lacked theoretical foundations, or do not scale. We present a learning-to-rank formulation that optimizes the fraction of satisfied users, with several scalable algorithms that explicitly takes document similarity and ranking context into account. Our formulation is a non-trivial common generalization of two multi-armed bandit models from the literature: \emph{ranked bandits} \citep{RBA-icml08} and \emph{Lipschitz bandits} \citep{LipschitzMAB-stoc08}. We present theoretical justifications for this approach, as well as a near-optimal algorithm. Our evaluation adds optimizations that improve empirical performance, and shows that our algorithms learn orders of magnitude more quickly than previous approaches.
\end{abstract}


\renewcommand{\thefootnote}{$^\fnsymbol{footnote}$}
\footnotetext[1]{Preliminary versions of this paper has been published as a conference paper in \emph{ICML 2010} and as a technical report at {\tt arxiv.org/abs/1005.5197} (May 2010). This version contains full proofs and a significantly revised presentation.}
\renewcommand{\thefootnote}{\arabic{footnote}}

\input{sec1-intro}

\input{sec-relatedWork}

\input{sec2-defns}

\input{sec3-model}

\input{sec4-algs}
\input{sec5-provable}
\input{sec6-eval}
\input{sec7-proofs}

\section{Further directions}
\label{sec:conclusions}

This paper initiates the study of bandit learning-to-rank with side information on similarity between documents, focusing on an idealized model of document similarity based on the new notion of ``conditional Lipschitz-continuity''.
As discussed in Section~\ref{sec:provable}, we conjecture that provable performance guarantees can be improved significantly.
On the experimental side, future work will include evaluating the model on web search data, and designing sufficiently memory- and time-efficient implementations to allow experiments on real users. An interesting challenge in such an endeavor would be to come up with effective similarity measures.  A natural next step would be to also exploit the similarity between search queries.


\bibliography{bib-abbrv,bib-icml10,bib-bandits,bib-slivkins}

\appendix
\input{sec-appendix.tex}

\end{document}

%% file: sec1-intro.tex
\section{Introduction}

Identifying the most relevant results to a query is a central problem in web search, hence learning ranking functions  has received a lot of attention \citep[e.g.,][]{Joachims-kdd02,Burges05,Chu05Gaussian,SoftRank}. One increasingly important goal is to learn from user interactions with search engines, such as clicks. We address the task of learning a ranking function that minimizes the likelihood of \emph{query abandonment}: the event that the user does not click on any of the search results for a given query. This objective is particularly interesting as query abandonment is a major challenge in today's search engines, and is also sensitive to the diversity and redundancy among documents presented.


We consider the Multi-Armed Bandit (MAB) setting
\citep[e.g.\!][]{CesaBL-book}, which captures many online learning problems
wherein an algorithm chooses sequentially among a fixed set of alternatives, traditionally called ``arms''. In each round an algorithm chooses an arm and collects the corresponding reward. Crucially, the algorithm receives limited feedback -- only for the arm it has chosen, which gives rise to the tradeoff between \emph{exploration} (acquiring new information) and \emph{exploitation} (taking advantage of the information available so far).

While most of the literature on MAB corresponds to learning a single best alternative, MAB algorithms can also be extended to learning a ranking of documents that minimizes query abandonment~\citep{RBA-icml08,Streeter08}.
In this setting, called \emph{Ranked Bandits}, in each round an algorithm chooses an \emph{ordered list} of $k$ documents from some fixed collection of documents, and receives clicks on some of the chosen documents. Crucially, the click probability for a given document may depend on the  documents shown above: a user scrolls the list top-down and may leave as soon as she clicked on the first document. The goal is to minimize query abandonment.

\citet{RBA-icml08} and \citet{Streeter08} propose a simple but effective approach: for each position in the ranking there is a separate instance bandit algorithm  which is responsible for choosing a document for this position. However, the specific algorithms they considered are impractical at WWW scales.


Prior work on MAB algorithms has considered exploiting structure in the space of arms to improve convergence rates. One particular approach, articulated by \citet{LipschitzMAB-stoc08} is well suited to our scenario: when the arms form a metric space and the payoff function satisfies a Lipschitz condition with respect to this metric space. The metric space provides information about similarity between arms, which allows the algorithm to make inferences about
similar arms without exploring them. Further, they propose a  ``zooming algorithm" which partitions the metric space into regions (and treats each region as a ``meta-arm'') so that the partition is adaptively refined over time and becomes finer in regions with higher payoffs.

\OMIT{In web search, there are additional signals that can
make ``metric" approaches even more effective.  A search user typically scans results top down, and clicks on more relevant documents. One can therefore infer the {\em context} in which a
click happened: the skipped documents at higher ranks.}

In web search, a metric space directly models similarity between documents. (It is worth noting that most offline learning-to-rank approaches also rely on similarity between documents, at least implicitly.)

\OMIT{ 
Further, one can use additional signals. A search user typically scans results top down, and clicks on more relevant documents.
One can therefore infer the {\em context} in which a
click happened: the skipped documents at higher ranks. To
fully exploit the context we factor in both {\em conditional
clickthrough rates} and {\em correlated clicks}.  The former conditions on the event that the user skipped a set of documents \citep[as suggested by][]{chen06}, and the
latter refers to the probability that two documents are both relevant or both irrelevant to a given user.
} 

\xhdr{Our contributions.}
This paper initiates the study of bandit learning-to-rank with side information on similarity between documents. We adopt the Ranked bandits setup: a user scrolls the results top-down and may leave after a single click, the goal is to minimize query abandonment. The similarity information is expressed as a metric space.

In this paper we consider a ``perfect world'' scenario: there exists an informative distance function which meaningfully describes similarity between documents in a ranked setting, and an algorithm has access to such function. We focus on two high-level questions: How to represent the knowledge of document similarity, and how to use it algorithmically in a bandit setting. We believe that studying such ``perfect world'' scenario is useful, and perhaps necessary, to inform and guide the corresponding data-driven work.

We propose a simple bandit model which combines \emph{Ranked bandits}
~\citep{RBA-icml08} and \emph{Lipschitz bandits}~\citep{LipschitzMAB-stoc08}, and admits efficient bandit algorithms that, unlike those in prior work on bandit learning-to-rank, scale to large document collections. Our model is based on the new notion of ``conditional Lipschitz continuity'' which asserts that similar documents have similar click probabilities even conditional on the event that all documents in a given set of documents are skipped (i.e., not clicked on) by the current user. We study this model both theoretically and empirically.

First, we validate the expressiveness of our model by providing an explicit construction for a wide family of plausible user distributions which provably fit the model. The analysis of this construction is perhaps the most technical contribution of this paper. We also use this construction in simulations.


Second, we put forth a battery of algorithms for our model. Some of these algorithms are straightforward combinations of ideas from prior work on Ranked bandits and Lipschitz bandits, and some are new.

A crucial insight in the new algorithms is that for each position $i$ in the ranking there is a \emph{context} that we can use, namely the set of documents chosen for the above positions in the same round. Indeed, since our objective is non-abandonment we only care about position $i$ if all documents shown above $i$ have been skipped in the present round. So the algorithm responsible for position $i$ can simply \emph{assume} that these documents have been skipped.

This interpretation of contexts allows us to cast the position-$i$ problem as a \emph{contextual bandit} problem. Moreover, we derive a Lipschitz condition on contexts (with respect to a suitably defined metric), which allows us to use the contextual Lipschitz MAB machinery from~\citep{contextualMAB-slivkins09}. We also exploit correlations between clicks: if a given document is included in the context -- i.e., if this document is skipped by the current user -- then similar documents are likely to be skipped, too. More specifically, we propose two algorithms that use contexts: a ``heavy-weight'' algorithm  which uses both the metric on contexts and correlated clicks, and a ``light-weight'' algorithm which uses correlated clicks but not the metric on contexts.

Third, we provide scalability guarantees for the heavy-weight contextual algorithm, proving that the convergence rate depends only on the dimensionality of the metric space but not on the number of documents. However, we argue that our provable guarantees do not fully reflect the power of the algorithm, and outline some directions for the follow-up theoretical work. In particular, we identify a stronger benchmark and discuss convergence to this benchmark. We provide an initial result: we prove, without any guarantees on the convergence rate, that the heavy-weight contextual algorithm indeed converges to this stonger benchmark. This theoretical discussion is one of the contributions.

Finally, we empirically study the performance of our algorithms. We run a large-scale simulation using the above-mentioned construction with realistic parameters. The main goal is to compare the convergence rates of the various approaches.  In particular, we confirm that metric-aware algorithms significantly outperform the metric-oblivious ones, and that taking the context into account improves the convergence rate. Somewhat surprisingly, our light-weight contextual algorithm performs better than the heavy-weight one.

A secondary, smaller-scale experiment studies the limit behaviour of the algorithms, i.e. the query abandonment probability that the algorithms converge to. Following the theoretical discussion mentioned above, we design a principled example on which different algorithms exhibit very different limit behaviour. Interestingly, the heavy-weight contextual algorithm is the only algorithm that achieves the optimal limit behaviour in this experiment.

\xhdr{Map of the paper.} We start with a brief survey of related work (Section~\ref{sec:related-work}). We define our model in Section~\ref{sec:defns}, and validate its expressiveness in Section~\ref{sec:model}. In-depth discussion of relevant approaches from prior work is in Section~\ref{sec:algorithms-prior}. Our new approach, ranked contextual bandits in metric spaces, is presented in Section~\ref{sec:ranked-contextual}. Scalability guarantees are discussed in Section~\ref{sec:provable}. We present our simulations in Section~\ref{sec:evaluation}.

To keep the flow of the paper, the lengthy proofs for the theoretical results in Section~\ref{sec:model} are presented in Section~\ref{sec:extending-mu} and Section~\ref{sec:proof-conditional}. Moreover, the background on instance-dependent regret bounds for \UCB-style algorithms is discussed in Appendix~\ref{app:instance-dependent}.

%% file: sec-relatedWork.tex
\section{Related work on multi-armed bandits}
\label{sec:related-work}

Multi-armed bandits has been studied for many decades as a simple yet expressive model for understanding exploration-exploitation tradeoffs. A thorough discussion of the literature on bandit problems is beyond the scope of this paper. For background, a reader can refer to a book~\citep{CesaBL-book} and a recent survey~\citep{Bubeck-survey12} on regret-minimizing bandits.%
\footnote{Regret of an algorithm in $T$ rounds, typically denoted $R(T)$, is the expected payoff of the benchmark in $T$ rounds minus that of the algorithm. A standard benchmark is the best arm in hindsight.}
A somewhat different, Bayesian perspective can be found in surveys~\citep{Sundaram-survey03,Bergemann-survey06}.

\newcounter{FnIndWork}
\addtocounter{footnote}{1}
\setcounter{FnIndWork}{\value{footnote}}
\newcommand{\IndWork}{\footnotemark[\value{FnIndWork}]}

On a very high level, there is a crucial distinction between regret-minimizing formulations and Bayesian/MDP formulations (see the surveys mentioned above); this paper follows the former.
Among regret-minimizing formulations, an important distinction is between stochastic rewards~\citep{Lai-Robbins-85,bandits-ucb1} and adversarial rewards~\citep{bandits-exp3}.

Below we survey several directions that are directly relevant to this paper.

\footnotetext[\value{footnote}]{This is either concurrent or subsequent work with respect to the conference publication of this paper.}

\xhdr{Ranked bandits.}
A bandit model in which an algorithm learns a ranking of documents with a goal to minimize query abandonment has been introduced in~\citep{RBA-icml08} under the name \emph{ranked bandits}. A crucial feature in this setting is that the click probability for a given document may depend not only on the document and the position in which it is shown, but also the  documents shown above. In particular, documents shown above can ``steal'' clicks from the documents shown below, in the sense that a user scrolls the list top-down and may leave as soon as she clicked on the first document.

Independently,~\citep{Streeter08} considered a more general model where the goal is to minimize an arbitrary (known) submodular set function, rather than query abandonment. A further generalization to submodular functions on ordered assignements (rather than on sets) was considered in~\citep{Golovin09}. The contributions of the three papers essentially coincide for the special case of ranked bandits.

\citep{Uchiya-alt10}\IndWork{} and \citep{Kale-slate-nips10}\IndWork{} considered a related bandit model in which an algorithm selects a ranking of documents in each round, but the click probabilities for a given document do not depend on which other documents are shown to the same user.

\xhdr{Bandits with structure.}
Numerous papers enriched the basic MAB setting by assuming some structure on arms, typically in order to handle settings where the number of arms is very large or infinite. Most relevant to this paper is the model where arms lie in a metric space and their expected rewards satisfy the Lipschitz condition with respect to this metric space (see Section~\ref{sec:defns} for details). This model, for a general metric space, has been introduced in~\cite{LipschitzMAB-stoc08} under the name \emph{Lipschitz MAB}; the special case of unit interval has been studied in~\citep{agrawal-bandits-95,Bobby-nips04,AuerOS/07} under the name \emph{continuum-armed bandits}. Subsequent work on Lipschtz MAB includes~\citep{xbandits-nips08,DichotomyMAB-soda10,Munos-ecml10,contextualMAB-slivkins09,ImplicitMAB-nips11}. A closely related model posits that arms corresponds to leaves on a tree, but no metric space is revealed to the algorithm ~\citep{Kocsis-ecml06,yahoo-bandits07,Munos-uai07,ImplicitMAB-nips11}.

Another commonly assumed structure is linear or convex payoffs ~\citep[e.g.][]{Bobby-stoc04,FlaxmanKM-soda05,DaniHK-nips07,AbernethyHR-colt08,Hazan-soda09}.
Linear/convex payoffs is a much stronger assumption than similarity, essentially because it  allows to make strong inferences about far-away arms. Other structural assumptions have been considered, e.g.~\citep{Munos-nips08} and \citep{Bubeck-colt10,GPbandits-icml10}\IndWork.

The distinction between the various possible structural assumptions is orthogonal to the distinction between stochastic and adversarial rewards. With a few exceptions, papers on MAB with linear/convex payoffs allow adversarial payoffs, whereas papers on MAB with similarity information focus on stochastic payoffs


\xhdr{Contextual bandits.} Here in each round the algorithm receives a \emph{context}, chooses an arm, and the reward depends both on the arm and the context. The term ``contextual bandits'' was coined in~\citep{Langford-nips07}. The setting, with a number of different modifications, has been introduced independently in several papers; a possibly incomplete list is \citep{Woodroofe79,Auer-focs00,Wang-sideMAB05,Langford-nips07,Hazan-colt07,yahoo-bandits07}.

There are several models for how contexts are related to rewards: rewards are linear in the context \citep[e.g.][]{Auer-focs00,Langford-nips07} and \cite{Reyzin-aistats11-linear}\IndWork, the context is a random variable correlated with rewards~\citep{Woodroofe79,Wang-sideMAB05,Zeevi-colt10}; rewards are Lipschitz with respect to a metric space on contexts~\citep{Hazan-colt07,contextualMAB-slivkins09} and~\citep{Pal-Bandits-aistats10}\IndWork.

 Most work on contextual bandits has been theoretical in nature; experimental work on contextual MAB includes~\citep{yahoo-bandits07} and \citep{Langford-www10,Langford-wsdm11}\IndWork.

%% file: sec2-defns.tex
\section{Problem formalization: ranked bandits in metric spaces}
\label{sec:defns}

Let us introduce the online learning-to-rank problem that we study in this paper.


\xhdr{Ranked bandits.}
Following \cite{RBA-icml08}, we are interested in learning an optimally diverse ranking of documents for a given query. We model it as a \emph{ranked bandit} problem as follows. Let $X$ be a set of documents (``arms"). Each `user' is represented by a binary \hlc{3}{\emph{relevance vector}}: a function $\pi: X\to \{0,1\}$. A document $x\in X$ is called ``relevant" to the user if and only if $\pi(x)=1$. Let $\F_X$ be the set of all possible relevance vectors. Users come from a distribution $\P$ on $\F_X$ that is fixed but not revealed to an algorithm.\footnote{This also models users for whom documents are probabilistically relevant \citep{RBA-icml08}.} This $\P$ will henceforth be called the \emph{user distribution}.

In each round, the following happens: a user arrives, sampled independently from $\P$; an algorithm outputs a list of $k$ documents; the user scans this list top-down, and clicks on the first relevant document. The goal is to maximize the expected fraction of \emph{satisfied users}: users who click on at least one document.
Note that in contrast with prior work on diversifying existing rankings \citep[e.g.\!][]{Carbonell98MMR}, the algorithm needs to directly learn a diverse ranking.

Since we count satisfied users rather than the clicks themselves, we can assume w.l.o.g. that a user leaves once she clicks once. (Alternatively, the algorithm does not record any subsequent clicks.) A user is satisfied or not satisfied independently of the order in which she scans the results. However, the assumption of the top-down scan determines the feedback received by the algorithm, i.e. which document gets clicked.

We will say that there are $k$ \emph{slots} to be filled in each round, so that when the algorithm outputs the list of $k$ documents, the $i$-th document in this list appears in slot $i$. Note that the standard model of MAB with stochastic rewards ~\citep[e.g.\!][]{bandits-ucb1} is a special case with a single slot ($k=1$).


\xhdr{Click probabilities.}
Recall that $\P$ is a distribution over relevance vectors. The \emph{pointwise mean} of $\P$ is a function $\mu:X\to [0,1]$ such that
	$\mu(x) \triangleq \E_{\pi\sim \P}[\pi(x)]$.
Thus, $\mu(x)$ is the click probability for document $x$ if it appears in the top slot.

Each slot $i>1$ is examined by the user only in the event that all documents in the higher slots are not clicked, so the relevant click probabilities for this slot are conditional on this event. Formally, fix a subset of documents $S\subset X$ and let
	$Z_S \triangleq \{ \pi(\cdot) = 0 \text{ on $S$} \} $
be the event that all documents in $S$ are not relevant to the user. Let $(\P|Z_S)$ be the distribution of users obtained by conditioning $\P$ on this event, and let $\mu(\cdot\,|Z_S)$ be its pointwise mean. Then $\mu(x\,|Z_S)$ is the click probability for document $x$ if $S$ is the set of documents shown above $x$ in the same round.

\xhdr{Metric spaces.}
Throughout the paper, let  $(X,\D)$ be a \emph{metric space}.That is, $X$ is a set and $\D$ is a symmetric function on $X\times X\to [0,\infty]$ such that $\D(x,y)=0 \iff x=y$, and
$\D(x,y) + \D(y,z) \geq \D(x,z)$ (triangle inequality).

A function $\nu:X\to \R$ is said to be \emph{Lipschitz-continuous} with respect to $(X,\D)$ if
\begin{align}\label{eq:def-Lip}
	|\nu(x)-\nu(y)| \leq \D(x,y) \qquad \text{for all $x,y\in X$}.
\end{align}
Throughout the paper, we will write \emph{L-continuous} for brevity.

A user distribution $\P$ is called L-continuous with respect to~$(X,\D)$ if its pointwise mean $\mu$ is L-continuous with respect to~$(X,\D)$.

\xhdr{Document similarity.}
To allow us to incorporate information about similarity between documents, we start with the model, called \emph{Lipschitz MAB}, proposed by~\cite{LipschitzMAB-stoc08} for the standard (single-slot) bandits. In this model, an algorithm is given a metric space $(X,\D)$ with respect to which \hlc{2}{the pointwise mean $\mu$} is L-continuous.\footnote{One only needs to assume that similarity between any two documents $x,y$ is summarized by a number $\delta_{x,y}$ such that
	$|\mu(x)-\mu(y)| \leq \delta_{x,y}$.
Then one obtains a metric space by taking the shortest paths closure.}

While this model suffices for learning the document at the top slot (see~\cite{LipschitzMAB-stoc08} for details), it is not sufficiently informative for lower slots. This is because the relevant click probabilities $\mu(\cdot\,|Z_S)$ are conditional and therefore are not directly constrained by L-continuity. To enable efficient learning in all $k$ slots, we will assume a stronger property called {\em  conditional L-continuity}:

\begin{definition}
$\P$ is called \emph{conditionally} L-continuous w.r.t. $(X,\D)$ if the conditional pointwise mean $\mu(\cdot|Z_S)$ is L-continuous for all $S\subset X$.
\end{definition}

\OMIT{
  Distribution $\P$ is called \emph{conditionally Lipschitz-continuous} w.r.t. $(X,\D)$ given $S\subset X$ if the conditional pointwise mean $\mu(\cdot|S)$ is Lipschitz-continuous w.r.t. $(X,D)$.
\begin{align}\label{eq:def-cond-Lip-correlated}
\Pr_{\pi\sim \P}[\pi(x)\neq \pi(y) \;|\; \pi_{|S} = 0] \leq \D(x,y),
\end{align}
for all $x,y\in X$ and all $S\subset X\setminus\{x,y\}$,
}

Now, a document $x$ in slot $i>1$ is examined only if event $Z_S$ happens, where $S$ is the set of documents in the higher slots. $x$ has a conditional click probability $\mu(x|Z_S)$. The function $\mu(\cdot\,| Z_S)$ satisfies the Lipschitz condition~\refeq{eq:def-Lip}, which will allow us to use the machinery from MAB problems on metric spaces.

Formally, we define the \emph{\problem}, an instance of which consists of a triple $(X,\D,\P)$, where $(X,\D)$ is a metric space that is known to an algorithm, and $\P$ is a latent user distribution which is conditionally L-continuous w.r.t. $(X,\D)$.

Note that the {\problem} subsumes the ``metric-free" ranked bandit problem from~\citet{RBA-icml08} (as a special case with a trivial metric space in which all distances are equal to $1$) and the Lipschitz MAB problem from~\citet{LipschitzMAB-stoc08} (as a special case with a single slot).

\subsection{Metric space: a running example}
\label{sec:defns-example}

Web documents are often classified into hierarchies, where closer pairs are more similar.\footnote{E.g., the Open Directory Project http://dmoz.org/} For evaluation, we assume the documents $X$ fall in such a tree, with each document $x\in X$ a leaf in the tree. On this tree, we consider a very natural metric: the distance between any two tree nodes $u,v$ is exponential in the height (i.e., the hop-count distance to the root) of their least common ancestor:
\begin{align}\label{eq:tree-metric}
\D(u,v) = c\times \eps^{\mathtt{height}(\mathtt{LCA}(u,v))},
\end{align}
for some constant $c$ and base $\eps \in (0,1)$. We call this the \emph{\eps-exponential tree metric} (with constant $c$). However, our algorithms and analyses extend to arbitrary metric spaces.

\subsection{Alternative notion of document similarity}
\label{sec:defns-correlated}

An alternative notion of document similarity focuses on \hlc{3}{\emph{correlated relevance}: correlation between the relevance of two documents to a given user}. We express ``similarity" by bounding the probability of the ``discorrelation event" $\{\pi(x)\neq \pi(y)\}$. Specifically, we consider {\em conditional L-correlation}, defined as follows:

\begin{definition}
Call $\P$ \emph{L-correlated} w.r.t. $(X,\D)$ if
\begin{align}\label{eq:def-L-correlated}
\Pr_{\pi\sim \P}\;
[\pi(x)\neq \pi(y)] \leq \D(x,y)
	\quad\forall x,y\in X.
\end{align}
Call $\P$ \emph{conditionally L-correlated} w.r.t. $(X,\D)$ if property~\refeq{eq:def-L-correlated} holds conditional on $Z_S$ for any $S\subset X$, i.e.
\begin{align*}
\Pr_{\pi\sim \text{\hlc{1}{$(\P|Z_S)$}}}\;
[\pi(x)\neq \pi(y)] \leq \D(x,y)
	\quad\forall x,y\in X, S\subset X.
\end{align*}
\end{definition}

It is easy to see that conditional L-correlation implies conditional L-continuity. \hlc{3}{In fact, we show that the two notions are essentially equivalent. Namely, we prove that} conditional L-continuity w.r.t.~$(X,\D)$ implies conditional L-correlation w.r.t.~$(X,2\D)$.

\begin{lemma}\label{lm:correlation-equivalence}
Consider an instance $(X,\D,\P)$ of the \problem. Then the user distribution $\P$ is conditionally L-correlated w.r.t. $(X,2\D)$.
\end{lemma}
\begin{proof}
Fix documents $x,y\in X$ and a subset $S\subset X$. For brevity, write ``$x=1$" to mean ``$\pi(x)=1$", etc. We claim that
\begin{align}\label{eq:corrXY}
\Pr[x=1 \wedge y=0 \,|Z_S] \leq \D(x,y).
\end{align}
Indeed, consider the event $Z = Z_{S+\{y\}}$. Applying the Bayes theorem to $(\P|Z_S)$, we obtain that
\begin{align}
\mu(x|Z)
	& = \Pr[x=1 \,|\, \{y=0\} \wedge Z_S] \nonumber \\
	& = \frac{\Pr[x=1 \wedge y=0 \,|Z_S]}{ \Pr[y=0 \,|Z_S] }. \label{eq:corrXY-Bayes}
\end{align}
On the other hand, since
	$\mu(y|Z)=0 $,
by conditional L-continuity it holds that
\begin{align}\label{eq:corrXY-Lip}
     \mu(x|Z) = |\mu(x|Z) - \mu(y|Z)| \leq \D(x,y),
\end{align}
so claim~\refeq{eq:corrXY} follows from~\eqref{eq:corrXY-Bayes} and~\eqref{eq:corrXY-Lip}.

Likewise,
    $ \Pr[x=0 \wedge y=1 \,|Z_S] \leq D(x,y)$.
Since
\begin{align*}
\{\pi(x) \neq \pi(y) \}
	= \{x=1 \wedge y=0\} \cup \{x=0 \wedge y=1\},
\end{align*}
it follows that
$\Pr[ \pi(x) \neq \pi(y) \,|Z_S] \leq  2\,D(x,y)$.
\end{proof}

%% file: sec3-model.tex
\section{Expressiveness of the model}
\label{sec:model}

\newcommand{\TT}{{\mathrm{T}}}

Our approach relies on the conditional L-continuity (equivalently, conditional L-correlation) of the user distribution. How ``expressive" is this assumption, i.e. how rich and ``interesting" is the collection of problem instances that satisfy it? While the unconditional L-continuity assumption is usually considered reasonable from the expressiveness point of view, even the unconditional L-correlation (let alone the conditional L-correlation) is a very non-trivial property about correlated relevance, and thus potentially problematic. A related concern is how to generate a suitable collection of problem instances for simulation experiments.

\OMIT{argue that conditional L-correlation is plausible in a realistic setting, and provide a family of user distributions to be used in experiments in Section~\ref{sec:evaluation}.}

We address both concerns by defining a natural (albeit highly stylized) generative model for the user distribution, which we then use in the experiments in Section~\ref{sec:evaluation}. We start with a tree metric space $(X,\D)$ and the desired pointwise mean $\mu:X\to (0,\tfrac12]$ that is L-continuous w.r.t. $(X,\D)$. The generative model provides a rich family of user distributions that are conditionally L-continuous w.r.t. $(X,c\,\D)$, for some small $c$. This result is a key theoretical contribution of this paper (and by far the most technical one).

We develop the generative model in Section~\ref{sec:model-tree}.  We extend this result to arbitrary metric spaces in Section~\ref{sec:model-metrics}, and to distributions over conditionally L-continuous user distributions in Section~\ref{sec:model-distributions}. To keep the flow of the paper, the detailed analysis is deferred to Section~\ref{sec:extending-mu} and Section~\ref{sec:proof-conditional}.

\subsection{Bayesian tree network}
\label{sec:model-tree}

The generative model is a tree-shaped Bayesian network \hlc{3}{with 0-1 ``relevance values" $\pi(\cdot)$ on nodes, where leaves correspond to documents.} The tree is essentially a topical taxonomy on documents: \hlc{3}{subtopics correspond to subtrees. The relevance value} on each sub-topic is obtained from that on the parent topic via a low-probability mutation.

The mutation probabilities need to be chosen so as to guarantee conditional L-continuity and the desired pointwise mean $\mu$. It is fairly easy to derive a necessary and sufficient condition for the pointwise mean, and a necessary condition for conditional L-continuity. The latter condition states that the mutation probabilities need to be bounded in terms of the distance between the child and the parent. The hard part is to prove that this condition is \emph{sufficient}.

Let us describe our Bayesian tree network in detail. The network inputs a tree metric space $(X,\D)$ and the desired pointwise mean $\mu$, and outputs a relevance vector $\pi: X\to \{0,1\}$. Specifically, we assume that documents are leaves of a finite rooted edge-weighted tree, which we denote $\DocTree$, with node set $V$ and leaf set $X\subset V$, so that $\D$ is a (weighted) shortest-paths metric on $V$.

Recall that $\mu$ is L-continuous w.r.t. $(X,\D)$. We assume that $\mu$ takes values in the interval $[\alpha,\tfrac12]$, for some constant parameter $\alpha>0$. We show that $\mu$ can be extended from $X$ to $V$ preserving the range and L-continuity (see Section~\ref{sec:extending-mu} for the proof).

\begin{lemma}\label{lm:extension-mu-short}
$\mu$ can be extended to $V$ so that $\mu:V \to [\alpha,\tfrac12]$ is L-continuous w.r.t. $(V,\D)$.
\end{lemma}

In what follows, by a slight abuse of notation we will assume that the domain of $\mu$ is $V$, with the same range
$[\alpha,\tfrac12]$, and that $\mu$ is L-continuous w.r.t. $(V,\D)$. Also, we redefine the relevance vectors to be functions
    $V\to \{0,1\}$
rather than $X\to \{0,1\}$.

The Bayesian network itself is very intuitive. We pick $\pi(\text{root})\in \{0,1\}$ at random with a suitable expectation $\mu(\text{root})$, and then proceed top-down so that the child's bit is obtained from the parent's bit via a low-probability mutation. The mutation is parameterized by functions $q_0, q_1: V\to [0,1]$, as described in Algorithm~\ref{alg:users}: for each node $u$, if the parent's bit is set to $b$ then the mutation $\{\pi(u)=1-b\}$ happens with probability $q_b(u)$. These parameters let us vary the degree of independence between each child and its parent, resulting in a rich family of user distributions.

\OMIT{
Moreover, our construction can provide limited independence between clicks on faraway points in the metric space, the ``amount" of such independence being proportional to the distance.\footnote{One may want to model negative correlation between clicks on documents from very distinct topics. Then one could first choose among the topics, then sample from the corresponding topic-specific user distribution.}
}

\newcommand{\algTab}{\hspace{3mm}}

\begin{algorithm}[t]
\begin{algorithmic}
\caption{User distribution for tree metrics}
\label{alg:users}

\STATE {\bf Input:}
	Tree (root $r$, node set $V$); $\mu(r) \in [0,1]$
\STATE\algTab mutation probabilities $q_0,q_1:V\to [0,1]$
\STATE {\bf Output:}
	relevance vector $\pi: V\to \{0,1\}$
\vspace{2mm}

\FUNCTION{AssignClicks(tree node $v$)}
\STATE $b\leftarrow \pi(v)$
\FOR{each child $u$ of $v$}
\STATE $\pi(u) \leftarrow
	\begin{cases}
		1-b   & \text{w/prob $q_b(u)$}\\
		b & \text{otherwise}
	\end{cases}$
\STATE AssignClicks(u)
\ENDFOR
\ENDFUNCTION
\vspace{2mm}

\STATE Pick $\pi(r)\in\{0,1\}$ at random with expectation $\mu(r)$
\STATE AssignClicks(r)
\end{algorithmic}
\end{algorithm}

To complete the construction, it remains to define the mutation probabilities $q_0, q_1$. Let $\P$ be the resulting user distribution. It is easy to see that $\mu$ is the pointwise mean of $\P$ on $V$ if and only if
\begin{align}\label{eq:construction-consistence}
\mu(u)  = (1-\mu(v))\, q_0(u)+ \mu(v)(1-q_1(u))
\end{align}
whenever $u$ is a child of $v$. (For sufficiency, use induction on the tree.) Further, letting $q_b = q_b(u)$ for each bit $b\in\{0,1\}$, note that
\begin{align*}
\Pr[\pi(u) \neq \pi(v)]
	&= \mu(v)\, q_1 + (1-\mu(v))\, q_0 \\
	&= \mu(v) (q_0+q_1) + (1-2\mu(v))\, q_0 \\
	&\geq \mu(v) (q_0+q_1).
\end{align*}
Thus, if $\P$ is L-correlated w.r.t. $(X,\D)$ then
\begin{align}\label{eq:construction-small}
q_0(u)+ q_1(u) \leq \D(u,v)/\mu(v).
\end{align}
We show that (\ref{eq:construction-consistence}-\ref{eq:construction-small}) suffices to guarantee conditional L-continuity.

For a concrete example, one could define
\begin{align}\label{eq:model-q}
(q_0(u),\, q_1(u)) = \begin{cases}
	\left(0, \tfrac{\mu(v)-\mu(u)}{\mu(v)} \right)		  & \hspace{-2mm}\text{if  $\mu(v)\geq \mu(u)$} \\
	\left( \tfrac{\mu(u)-\mu(v)}{1-\mu(v)}, \; 0 \right)  & \hspace{-2mm}\text{otherwise}.
\end{cases}
\end{align}
The $q_0,q_1$ defined as above satisfy (\ref{eq:construction-consistence}-\ref{eq:construction-small}) for any $\mu$ that is L-continuous on $(V,\D)$.

The provable properties of Algorithm~\ref{alg:users} \hlc{3}{are summarized in the theorem below. It is technically more convenient to state this theorem in terms of L-correlation rather than L-continuity.}

\begin{theorem}\label{thm:expressiveness}
Let $\D$ be the shortest-paths metric of an edge-weighted rooted tree with a finite leaf set $X$. Let
	$\mu:X\to [\alpha,\tfrac12]$, $\alpha>0$
be L-continuous w.r.t. $(X,\D)$. Suppose
	$q_0,q_1:V\to[0,1]$
satisfy (\ref{eq:construction-consistence}-\ref{eq:construction-small}).

Let $\P$ be the user distribution constructed by Algorithm~\ref{alg:users}. Then $\P$ has pointwise mean $\mu$ and is conditionally $L$-correlated w.r.t. $(X,3\,\D_\mu)$ where
\begin{align}\label{eq:thm-existence-tree-mu}
\D_\mu(x,y)
	\triangleq \D(x,y)\, \min\left(
				\tfrac{1}{\alpha}\,,\;
				\tfrac{3}{\mu(x)+\mu(y)}
	\right).
\end{align}
\end{theorem}

\begin{note}{Remark.}
The theorem can be strengthened by replacing $\D_\mu$ with the shortest-paths metric induced by $\D_\mu$.
\end{note}

Below we provide a proof sketch. The detailed proof is presented in Section~\ref{sec:proof-conditional}.

\begin{pfsketch}
As we noted above, the statement about the pointwise mean trivially follows from~\eqref{eq:construction-consistence} using induction on the tree. In what follows we focus on conditional L-correlation.

Fix leaves $x,y\in X$ and a subset $S\subset X$. Let $z$ be the least common ancestor of $x,y$. Recall that in Algorithm~\ref{alg:users} the bit $\pi(\cdot)$ at each node is a random mutation of that of its parent. We focus on the event $\mathcal{E}$ that no mutation happened on the $z\to x$ and $z\to y$ paths. Note that $\mathcal{E}$ implies
	$\pi(x) = \pi(y) = \pi(z)$. Therefore
\begin{align}\label{eq:body-expressivenes-failure}
	\Pr[\pi(x) \neq \pi(y) \,|Z_S] \leq \Pr[\bar{\mathcal{E}} \,|Z_S],
\end{align}
where $\bar{\mathcal{E}}$ is the negation of $\mathcal{E}$. Intuitively, $\bar{\mathcal{E}}$ is a  low-probability ``failure event". The rest of the proof is concerned with showing that
	$\Pr[\bar{\mathcal{E}} \,|Z_S] \leq 3\, \D_\mu(x,y)$.

First we handle the unconditional case. We claim that
\begin{align}\label{eq:body-expressivenes-unconditional}
	\Pr[\bar{\mathcal{E}}] \leq \D_\mu(x,y).
\end{align}
Note that~\eqref{eq:body-expressivenes-unconditional} immediately implies that $\P$ is L-correlated w.r.t. $(X,\D_\mu)$. This claim is not very difficult to prove, essentially since the condition~\refeq{eq:construction-small} is specifically engineered to satisfy the unconditional L-correlation property. We provide the proof in detail.

Let
	$w \in \argmin_{u\in P_{xy}} \mu(u)$,
where $P_{xy}$ is the $x\to y$ path.
Let $(z=x_0, x_1,\,\ldots\,, x_n=x)$ be the $z\to x$ path. For each $i\geq 1$ by~\eqref{eq:construction-small} the probability of having a mutation at $x_i$ is at most $\D(x_i, x_{i-1}) /\mu(w)$, so the probability of having a mutation on the $z\to x$ path is at most $\D(x,z)/\mu(w)$. Likewise for the $z\to y$ path. So
$\Pr[\bar{\mathcal{E}}]
	\leq \D(x,y)/\mu(w)
	\leq \D(x,y)/ \alpha$.

It remains to prove that
\begin{align}\label{eq:body-expressiveness-eq1}
\Pr[\bar{\mathcal{E}}] \leq D(x,y)\; \tfrac{3}{\mu(x)+\mu(y)}.
\end{align}
Indeed, by L-continuity it holds that
\begin{align*}
\mu(w) &\geq \mu(x) - \D(x,w),\\
\mu(w) &\geq \mu(y) - \D(y,w).
\end{align*}
Since $\D(x,y) = \D(x,w)+\D(y,w)$, it follows that
\begin{align}\label{eq:body-expressiveness-eq2}
\mu(w) &\geq \tfrac{\mu(x) + \mu(y) - \D(x,y)}{2}.
\end{align}
Now, either the right-hand side of~\eqref{eq:body-expressiveness-eq2} is at least $\tfrac{\mu(x) + \mu(y)}{3}$, or the right-hand side of~\eqref{eq:body-expressiveness-eq1} is at least $1$. In both cases~\eqref{eq:body-expressiveness-eq1} holds. This completes the proof of the claim~\refeq{eq:body-expressivenes-unconditional}.

The conditional case is much more difficult. We handle it by showing that
\begin{align}\label{eq:app-expressiveness-conditional}
\Pr[\bar{\mathcal{E}} \,|\, Z_S] \leq 3\, \Pr[\bar{\mathcal{E}}].
\end{align}
In fact,~\eqref{eq:app-expressiveness-conditional} holds even if~\eqref{eq:construction-small} is replaced with a much weaker bound:
	$\max(q_0(u),\, q_1(u)) \leq \tfrac12$
for each $u$.

The mathematically subtle proof of~\eqref{eq:app-expressiveness-conditional} can be found in Section~\ref{sec:proof-conditional}. The crux in this proof is that event $Z_S$ is more likely if document $z$ is not relevant to the user:
\begin{align*}
	\Pr[Z_S \,|\, z=0] \geq \Pr[Z_S \,|\, z=1]. \qquad\qedhere
\end{align*}
\end{pfsketch}

\subsection{Arbitrary metric spaces}
\label{sec:model-metrics}

\newcommand{\PT}{\P_{\text{tree}}}

We can extend Theorem 3.1 to arbitrary metric spaces using prior work on \emph{metric embeddings}. Fix an $N$-point metric space $(X,\D)$ and a function $\mu:X\to [\alpha,\tfrac12]$ that is L-continuous on $(X,\D)$. It is known~\citep{Bar96,FRT03} that  there exists a distribution $\PT$ over tree metric spaces $(X,\mathcal{T})$ such that
	$\D(x,y)\leq  \mathcal{T}(x,y)$
and
\begin{align*}
	\E_{\mathcal{T}\sim \PT}\, [\mathcal{T}(x,y)] \leq c\,\D(x,y)
		\quad\forall x,y\in X,
\end{align*}
where $c = O(\log N)$.%
\footnote{This is the main result in~\citet{FRT03}, which improves on an earlier result in \citet{Bar96} with $c = O(\log^2 N)$. For point sets in a $d$-dimensional Euclidean space one could take $c = O(d\log\tfrac{1}{\eps})$, where $\eps$ is the minimal distance. In fact, this result extends to a much more general family of metric spaces -- those of doubling dimension $d$~\citep{Gup03}. Doubling dimension, the smallest $d$ such that any ball can be covered by $2^d$ balls of half the radius, has been introduced to the theoretical computer science literature in~\citet{Gup03}, and has been a well-studied concept since then.}

Our construction (Algorithm~\ref{alg:users-general-metric}) is simple: first sample a tree metric space $(X,\mathcal{T})$ from $\PT$, then independently generate a user distribution $\P_\mathcal{T}$ for $(X,\mathcal{T})$ as per Algorithm 1.

\begin{algorithm}[t]
\begin{algorithmic}
\caption{User distribution for arbitrary metric spaces}
\label{alg:users-general-metric}

\STATE {\bf Input:} metric space $(X,\D)$;
function $\mu:X\to [\alpha,\tfrac12]$ that is L-continuous on $(X,\D)$.

\STATE {\bf Output:}
	relevance vector $\pi: X\to \{0,1\}$
\vspace{2mm}

\STATE 1. Sample a tree metric space $(X,\mathcal{T})$ from $\PT$,
\STATE 2. Run Algorithm~\ref{alg:users} for $(X,\mathcal{T})$, output the resulting $\pi$.
\end{algorithmic}
\end{algorithm}

\begin{theorem}
The user distribution $\P$ produced by Algorithm~\ref{alg:users-general-metric} has pointwise mean $\mu$ and is conditionally L-correlated w.r.t. $(X,3c\,\D_\mu)$, where $\D_\mu$ is given by
\begin{align*}
D_\mu(x,y)
	= D(x,y)\, \min\left(\tfrac{1}{\alpha}\,,\;
				\tfrac{3}{\mu(x)+\mu(y)}
	\right).
\end{align*}
\end{theorem}

\begin{proof}
The function $\mu$ is L-continuous w.r.t. each tree metric space $(X,\mathcal{T})$, so
by Theorem 3.1 user distribution $\P_\mathcal{T}$ has pointwise mean $\mu$ and is conditionally L-correlated w.r.t. $(X,3\, \mathcal{T}_\mu)$. It follows that the aggregate user distribution $\P$ has pointwise mean $\mu$, and moreover for any $x,y\in X$ and $S\subset X$ we have
\begin{align*}
&\Pr_{\pi\sim \P}\,[\pi(x)\neq \pi(y) \,|Z_S] \\
&\qquad \leq
	\E_{\mathcal{T}\sim \PT}\, \left[
			\Pr_{\pi\sim \P_\mathcal{T}}[\pi(x)\neq \pi(y)\,|Z_S]
		\right] \\
&\qquad \leq \E_{\mathcal{T}\sim \PT}\, \left[ 3\,  \mathcal{T}_\mu(x,y) \right] \\
&\qquad	\leq 3c\,\D_\mu(x,y). \qquad \qedhere
\end{align*}
\end{proof}

\subsection{Distributions over user distributions}
\label{sec:model-distributions}

\ASnote{3}{AS: This subsection is new.}

Let us verify that conditional L-continuity is \emph{robust}, in the sense that any distribution over conditionally L-continuous user distributions is itself conditionally L-continuous. This result considerably extends the family of user distributions for which we have conditional L-continuity guarantees.

\begin{lemma}\label{lm:affine}
Let $\P$ be a distribution over countably many user distributions $\P_i$ that are conditionally L-continuous w.r.t. a metric space $(X,\D)$. Then $\P$ is conditionally L-continuous w.r.t. $(X,\D)$.
\end{lemma}

\begin{proof}
Let $\mu$ and $\mu_i$ be the (conditional) pointwise means of $\P$ and $\P_i$, respectively. Formally, let us treat each $\P_i$ as a measure, so that $\P_i(E)$ is the probability of event $E$ under $\P_i$. Let
	$ \P = \sum_i q_i\, \P_i$,
where $\{q_i\}$ are positive coefficients that sum up to $1$. Fix documents $x,y\in X$ and a subset $S\subset X$. Then
\begin{align*}
\mu(x|S)
	&= \P(x=1 \,|Z_S)
	= \frac{\P(x=1 \wedge Z_S)}{\P(Z_S)} \\
	&= \frac{ \sum_i q_i\; \P_i(x=1 \wedge Z_S)}{\P(Z_S)} \\
	&= \frac{ \sum_i q_i\; \P_i(Z_S)\; \mu_i(x|Z_S)}{\P(Z_S)}.
\end{align*}
It follows that
\begin{align*}
&|\mu(x|S) - \mu(y|S)| \\
&\qquad = \frac{ \sum_i q_i\; \P_i(Z_S)\; (\mu_i(x|Z_S) - \mu_i(y|Z_S))}{\P(Z_S)} \\
&\qquad \leq \frac{ \sum_i q_i\; \P_i(Z_S)\; D(x,y)}{\P(Z_S)} \\
&\qquad \leq D(x,y). \qedhere
\end{align*}
\end{proof}

%% file: sec4-algs.tex
\section{Algorithms from prior work}
\label{sec:algorithms-prior}



\renewcommand{\algorithmiccomment}[1]{\hspace\fill {\it \tiny #1}}

\OMIT{The ranked bandit approach from \citet{RBA-icml08} and the metric bandit algorithms from~\citet{Bobby-nips04, LipschitzMAB-stoc08} can be used directly.
Then in Section~\ref{sec:contextual} we review the \emph{contextual bandit} setting and one particular algorithm for this setting. Finally, in Section~\ref{sec:ranked-contextual} we represent our setting as a (multi-slot) contextual MAB problem, and use this representation to design two new algorithms that select documents slot-by-slot (top-down), and explicitly treat the higher selected documents as a \emph{context} for the current slot.
}

Let us discuss some algorithmic ideas from prior work that can be adapted to our setting. Interestingly, one can combine these algorithms in a \emph{modular} way, which we make particularly transparent by putting forward a suitable naming scheme. Throughout this section, we let ${\tt Bandit}$ be some algorithm for the MAB problem.

\subsection{Ranked bandits}
\label{sec:prior-RBA}

Given some bandit algorithm ${\tt Bandit}$, the ``ranked" algorithm \rankedBandit{} for the multi-slot MAB problem is defined as follows \citep{RBA-icml08}. We have $k$ slots (i.e.,~ranks) for which we wish to find the best documents to present. In each slot $i$, a separate instance $\A_i$ of ${\tt Bandit}$ is created. In each round these instances select the documents to show independently of one another. If a user clicks on slot $i$, then this slot receives a reward of $1$, and all higher (i.e.,~skipped) slots $j<i$ receive a reward of $0$. For slots $j>i$, the state is rolled back as if this round had never happened (as if the user never considered these documents). If no slot is clicked, then all slots receive a reward of $0$.

Let us emphasize that the above approach can be applied to \emph{any} algorithm ${\tt Bandit}$. In~\cite{RBA-icml08}, this approach gives rise to algorithms \naiveUCB{} and \naiveEXP, based on MAB algorithms \UCB~and \EXP~(\citealp{bandits-ucb1}, \citeyear{bandits-exp3}). \EXP{} is designed for the \emph{adversarial} setting with no assumptions on how the clicks are generated, which translates into  concrete provable guarantees for \naiveEXP. \UCB{} is geared towards the \emph{stochastic} setting with i.i.d.~rewards on each arm, although\OMIT{and for this setting it (provably) achieves much better performance than \EXP. Since in the multi-slot MAB problem the users are sampled from an i.i.d.~source, using \naiveUCB{} is tempting. However,} the per-slot i.i.d.~assumption breaks for slots $i>1$ because of the influence of the higher slots. Nevertheless, in small-scale experiments \naiveUCB{} performs much better than \naiveEXP\OMIT{, and the latter is seen as overly pessimistic}~\citep{RBA-icml08}.

\xhdr{Provable guarantees.}
Letting $T$ be the number of rounds and $\texttt{OPT}$ be the probability of clicking on the optimal ranking, algorithm \rankedBandit{} achieves
\begin{align}\label{eq:regret-RBA}
	\E[\mathtt{\#clicks}] \geq (1-\tfrac{1}{e})\,T\times \OPT - k\, R(T),
\end{align}
where $R(T)$ is any upper bound on regret for {\tt Bandit} in each slot~\citep{RBA-icml08, Streeter08}.

In the multi-slot setting, \emph{performance} of an algorithm up to time $T$ is defined as the time-averaged expected total number of clicks. We will consider performance as a function of $T$. Assuming $R(T)=o(T)$ in~\eqref{eq:regret-RBA}, performance of \rankedBandit{} converges to or exceeds $(1-\tfrac{1}{e})\OPT$. Convergence to
	$(1-\tfrac{1}{e})\OPT$
is proved to be worst-case optimal. Thus, as long as $R(T)$ scales well with time, for the document collection sizes that are typical for the application at hand,
~\citet{RBA-icml08} interpret~\eqref{eq:regret-RBA} as a proof of an algorithm's scalability in the multi-slot MAB setting.

\rankedBandit{} is presented in~\cite{RBA-icml08} as the online version of the \emph{greedy algorithm}: an offline fully informed algorithm that selects documents greedily slot by slot from top to bottom. The performance of this algorithm is called the \emph{greedy optimum},%
\footnote{If due to ties there are multiple ``greedy rankings", define the greedy optimum via the \emph{worst} of them.}
which is equal to
	$(1-\tfrac{1}{e})\, \OPT$
in the worst case, but for ``benign" problem instances it can be as good as $\mathtt{OPT}$. The greedy optimum is a more natural benchmark for \rankedBandit{} than $(1-\tfrac{1}{e})\OPT$. However, results w.r.t. this benchmark are absent in the literature.%
\footnote{Following the conference publication of this paper, Streeter and Golovin claimed that the techniques in~\cite{Streeter08} can be used to extend~\eqref{eq:regret-RBA} to the greedy optimum benchmark. If so, then it may be possible to use the same approach to improve our guarantees.}

\subsection{Lipschitz bandits}

Both \UCB{} and \EXP{} are impractical when there are too many documents to explore them all. To alleviate this issue, one can use the similarity information provided by the metric space and the Lipschitz assumption; this setting is called \emph{Lipschitz MAB}.

Below we describe two ``metric-aware" algorithms from \citep{Bobby-nips04} and \citep{LipschitzMAB-stoc08}. Both
are well-defined for arbitrary metric spaces, but for simplicity we present them for a special case in which documents are leaves in a \emph{document tree} (denoted $\DocTree$) with an \eps-exponential tree metric. In both algorithms, a \emph{subtree} is chosen in each round, then a document in this subtree is sampled at random, choosing uniformly at each branch.

Given some bandit algorithm ${\tt Bandit}$,~\citet{Bobby-nips04} define algorithm ${\tt GridBandit}$ for the Lipschitz MAB setting. This algorithm proceeds in phases: in phase $i$, the depth-$i$ subtrees are treated as ``arms", and a fresh copy of {\tt Bandit} is run on these arms.\footnote{As an empirical optimization, previous events can also be replayed to better initialize later phases.} Phase $i$ lasts for $k \eps^{-2i}$ rounds, where $k$ is the number of depth-$i$ subtrees. This meta-algorithm, coupled with an adversarial MAB algorithm such as \EXP, is the only algorithm in the literature that takes advantage of the metric space in the adversarial setting. Following~\cite{RBA-icml08}, we expect \gridEXP{} to be overly pessimistic for our problem, trumped by the corresponding stochastic MAB approaches such as \gridUCB.

\OMIT{From prior work~\cite{Bobby-nips04, LipschitzMAB-stoc08}, we know that the regret of \gridEXP{} (for the adversarial setting) and the regret of \gridUCB{} (for the time-invariant stochastic setting) is
\begin{align*}
	R(T)\leq O(T^{1-1/(d+2)} \log T).
\end{align*}
For \gridEXP{}, plugging this regret into the analysis in~\cite{RBA-icml08}, we obtain that for the $k$-slot Lipschitz-continuous MAB problem, the expected number of clicks is at least
	$(1-\tfrac{1}{e})\, \text{OPT} - k\, R(T)$.
}


\OMIT{ 

To improve empirical performance, we also use the following optimizations. First, the samples taken during each phase $i$ are replayed in the next phase $i+1$. Specifically, we pick $S_{i+1} \supset S_i$. For each $x\in S_i$, let $n_i(x)$ be the number of samples of $x$ in phase $i$. In phase $i+1$, the $j$-th  time algorithm {\tt Bandit} picks $x$, if $j\leq n_i(x)$ then instead of actually playing $x$ we return the $j$-th sample of $x$ from phase $i$. Second, when {\tt Bandit} chooses a given point $x$, then instead of playing a fixed $x$ we pick strategies at random from the allowed documents within the appropriate grid square, e.g. $B(x, \eps^2/2)$.
%
The analysis in~\cite{RBA-icml08} provides the following theorem for {\tt RankedGrid\EXP}:

\begin{theorem}\label{thm:RankedNaiveExp3}
For the $k$-slot Lipschitz-continuous MAB problem, the expected number of clicks obtained by {\tt RankedNaive\EXP} is at least
	$(1-\tfrac{1}{e})\, OPT - k\, R(T)$,
where $R(T)\leq O(T^{1-1/(d+2)} \log T)$.\footnote{Note that Theorem~\ref{thm:RankedNaiveExp3} does not require conditional Lipschitz-correlation: conditional Lipschitz-continuity (a weaker condition) suffices.}
\end{theorem}
\begin{proof}
TO BE WRITTEN... BUT IT IS SIMPLE (A.S.)
\end{proof}

} 

The ``zooming algorithm"~\citep[][Algorithm \ref{alg:zooming}]{LipschitzMAB-stoc08} is a more efficient version of \gridUCB: instead of iteratively reducing the grid size in the entire metric space, it \emph{adaptively} refines the grid in promising areas. It maintains a set $\mathcal{A}$ of \emph{active subtrees} which collectively partition the leaf set. In each round the active subtree with the maximal \emph{index} is chosen. The index of a subtree is (assuming stochastic rewards)
the best available upper confidence bound on the click probabilities
in this subtree. It is defined via the \emph{confidence radius}\footnote{The meaning of $\confRad(\cdot)$ is that w.h.p. the sample average is within $\pm \confRad(\cdot)$ from the true mean.} \hlc{1}{given (letting $T$ be the time horizon)} by
\begin{align}\label{eq:conf-rad}
	\confRad(\cdot) \triangleq
		\sqrt{4 \log(T) /
			(1 + {\tt \#samples}(\cdot))}.
\end{align}
\OMIT{ where $T$ is the time horizon, and $n(\cdot)$ is the number of times the subtree has been chosen so far. }

The algorithm ``zooms in" on a given active subtree $u$ (de-activates $u$ and activates all its children) when $\confRad(u)$ becomes smaller than its \emph{width}
	$\size(u) \triangleq \epsilon^{\text{depth}(u)}
		= \max_{x,x'\in u} \D(x,x')$.

\OMIT{which is the maximal possible distance $\D(x,x')$ between any two documents  $x,x'\in u$.
The ``ranked zooming algorithm" will be denoted \zooming.}

\begin{algorithm}[t]
\begin{algorithmic}
\caption{``Zooming algorithm" in trees}
\label{alg:zooming}
\STATE {\bf initialize} (document tree $\DocTree$){\bf :}
\STATE \TAB	$\A\!\leftarrow\!\emptyset$; ~~activate($\texttt{root}(\DocTree)$)
	\vspace{2mm}
\STATE {\bf activate}(~$u\in\texttt{nodes}(\DocTree)$~){\bf :}
\STATE \TAB
		$\A\!\leftarrow\!\A\cup\{ u \};\;\;
		n(u)\!\leftarrow\!0;\;\;
		r(u)\!\leftarrow\!0$
	\vspace{2mm}
\STATE {\bf Main loop:}\vspace{-0.3mm}
\STATE \TAB $ u\leftarrow \argmax_{u\in \mathcal{A}}
                    \mathtt{index}(u)$,
\STATE \TAB\TAB where
					$\mathtt{index}(u) = \frac{r(u)}{n(u)}$ + 2\,$\confRad(u)$
\STATE \TAB ``Play" a random document from $\texttt{subtree}(u)$
\STATE \TAB $r(u) \leftarrow r(u) + \text{\{reward\}}$; $n(u) \leftarrow n(u) + 1$
\STATE \TAB {\bf if} $\confRad(u) < \size(u)$ {\bf then}
\STATE \TAB\TAB deactivate $u$: remove $u$ from $\mathcal{A}$
\STATE \TAB\TAB activate all children of $u$
\end{algorithmic}
\end{algorithm}

\OMIT{ 
We will consider two ``metric-oblivious" algorithms: \naiveUCB{} and \naiveEXP, and three ``metric-aware" ones: \rankedGridUCB, \rankedGridEXP, and \zooming.
} 

\xhdr{Provable guarantees.} Regret guarantees for the two algorithms above are independent of the number of arms (which, in particular, can be infinite). Instead, they depend on the covering properties of the metric space $(X,\D)$. A crucial notion here is the \emph{covering number} $N_r(X)$, defined as the minimal number of balls of radius $r$ sufficient to cover $X$. It is often useful to summarize the covering numbers $N_r(X)$, $r>0$ with a single number called the \emph{covering dimension}:
\begin{align}\label{eq:dim-defn}
\CovDim(X,\D) \triangleq \inf\{d\geq 0: N_r(X) \leq \alpha\, r^{-d} \quad\forall r>0 \}.
\end{align}
(Here $\alpha>0$ is a constant which we will keep implicit in the notation.)
In particular, for an arbitrary point set in $\R^d$ under the standard ($\ell_2$) distance, the covering dimension is $d$, for some $\alpha = O(1)$. For an \eps-exponential tree metric with maximal branching factor $b$, the covering dimension is $d = \log_{1/\eps}(b)$, with $\alpha=1$.

Against an oblivious adversary, \gridEXP{} has regret
\begin{align}\label{eq:regret-dim}
R(T) = \tilde{O}(\alpha\,T^{(d+1)/(d+2)}),
\end{align}
where $d$ is the covering dimension of $(X,\D)$.

For the stochastic setting, \gridUCB{} and the zooming algorithm enjoy strong instance-dependent regret guarantees. These guarantees reduce to~\eqref{eq:regret-dim} in the worst case, but are much better for ``nice'' problem instances. Informally, regret guarantees improve for problem instances in which the set of near-optimal arms has smaller covering numbers than the set of all arms. Regret guarantees for the zooming algorithm are (typically) much stronger than for \gridUCB. In particular, one can derive a version of \eqref{eq:regret-dim} with a different $d$ called the \emph{zooming dimension}, which is equal to the covering dimension in the worst case but can be much smaller, even $d=0$. These issues are further discussed in Appendix~\ref{app:instance-dependent}.

\subsection{Anytime guarantees and the doubling trick}
\label{sec:algs-anytime}

While the zooming algorithm, and also the contextual zooming algorithm from Section~\ref{sec:algorithms-prior-contextual}, are defined for a fixed time horizon, one can obtain the corresponding \emph{anytime} versions using a simple \emph{doubling trick}: in each phase $i\in \N$, run a fresh instance of the algorithm for $2^i$ rounds. These versions are run indefinitely and enjoy the same provable upper bounds on regret as the original algorithms (but now these bounds hold for each round).

\subsection{Ranked bandits in metric spaces}
\label{sec:prior-rankedMetric}

Using and combining the algorithms in the previous two subsections, we obtain the following battery of algorithms for \problem:

\begin{OneLiners}
\item metric-oblivious algorithms: $\naiveUCB$ and $\naiveEXP$.
\item simple metric-aware algorithms: $\rankedGridUCB$ and $\rankedGridEXP$\\ (ranked versions of $\gridUCB$ and $\gridEXP$, respectively).
\item $\zooming$: the ranked version of the zooming algorithm.
\end{OneLiners}

In theory, \rankedGridEXP{} scales to large document collections, in the sense that it achieves~\eqref{eq:regret-RBA} with $R(T)$ that does not degenerate with $\#$documents:

\begin{theorem}\label{thm:regret-RankGridEXP}
Consider the \problem{} on a metric space with covering dimension $d$ (as defined in~\eqref{eq:dim-defn}, with constant $\alpha$). Then after $T$ rounds $\rankedGridEXP$ achieves
\begin{align*}
	\frac{\E[\mathtt{\#clicks}]}{T} \geq (1-\tfrac{1}{e})\,\OPT -
            \tilde{O}\left( \frac{\alpha k}{T^{1/(d+2)}} \right).
\end{align*}
\end{theorem}

The theorem follows from the respective regret bounds for \gridEXP{} (\eqref{eq:regret-dim}) and \rankedBandit{} (\eqref{eq:regret-RBA}). We do not have any provable guarantees for other algorithms because the corresponding regret bounds for the single-slot setting do not directly plug into~\eqref{eq:regret-RBA}. However, the strong instance-dependent guarantees for $\gridUCB$ and especially for the zooming algorithm (even though they do not directly apply to the ranked bandit setting) suggest that $\rankedGridUCB$ and $\zooming$ are promising. We shall see that these two algorithms perform much better than $\rankedGridEXP$ in the experiments.

\subsection{Contextual Lipschitz bandits}
\label{sec:algorithms-prior-contextual}

We also leverage prior work on contextual bandits. The relevant contextual MAB setting, called contextual Lipschitz MAB, is as follows. In each round nature reveals a \emph{context} $h$, an algorithm chooses a document $x$, and the resulting reward is an independent $\{0,1\}$ sample with expectation $\mu(x|h)$. Further, one is given \emph{similarity information}: metrics $\D$ and $\DC$ on documents and contexts, respectively, such that for any two documents $x,x'$ and any two contexts $h,h'$ we have
\begin{align*}
	|\mu(x|h) - \mu(x'|h')| \leq \D(x,x') +\DC(h,h').
\end{align*}

Let $\XC$ be the set of contexts, and $\Xpairs = X\times \XC$ be the set of all (document, context) pairs. Abstractly, one considers the metric space 	 $(\Xpairs,\Dpairs)$, henceforth the \emph{DC-space}, where the metric is
\begin{align*}
	\Dpairs((x,h),\, (x',h')) = \D(x,x')+\DC(h,h').
\end{align*}

We will use the ``contextual zooming algorithm" (\contextualZooming) from~\cite{contextualMAB-slivkins09}. This algorithm is well-defined for arbitrary \hlc{4}{$\Dpairs$}, but for simplicity we will state it for the case when $\D$ and $\DC$ are \eps-exponential tree metrics.

\begin{algorithm}[t]
\begin{algorithmic}
\caption{\contextualZooming{} in trees}
\label{alg:contextualzooming}

\STATE {\bf initialize} (document tree $\DocTree$, context tree $\ContextTree$){\bf :}
\STATE \TAB	$\A \leftarrow \emptyset$;
		~~activate(~$\texttt{root}(\DocTree),\, \texttt{root}(\ContextTree)$~)
		\vspace{2mm}
\STATE {\bf activate}
		(~$u\in \texttt{nodes}(\DocTree),\; \contextU\in\texttt{nodes}(\ContextTree)$~){\bf :}
\STATE \TAB
		$\A \leftarrow \A\cup\{ (u,\contextU) \};\;\;
		n(u,\contextU) \leftarrow 0;\;\;
		r(u,\contextU) \leftarrow 0$
	\vspace{2mm}
\STATE {\bf Main loop:}\vspace{-0.3mm}
\STATE \TAB Input a context $h\in \texttt{nodes}(\ContextTree)$
\STATE \TAB
	$\displaystyle
		(u,\contextU) \leftarrow
			\argmax_{(u,\, \contextU)\in \mathcal{A}:\; h \in \contextU} \;
            \mathtt{index}(u,\contextU),$
\STATE\TAB\TAB
    where $ \mathtt{index}(u,\contextU) =
			\size(u \times \contextU)  +
			\frac{r(u,\contextU)}{n(u,\contextU)}
			+ \confRad(u,\contextU)$\vspace{1mm}
\STATE \TAB ``Play" a random document from $\texttt{subtree}(u)$
\STATE \TAB
	$r(u, \contextU)  \leftarrow  r(u, \contextU)  +  \text{\{reward\}}$;\;
	$n(u, \contextU)  \leftarrow  n(u, \contextU)  +  1$
\STATE \TAB {\bf if} \confRad$(u,\contextU) < \size(u,\contextU)$ {\bf then}
\STATE \TAB\TAB deactivate $(u,\contextU)$: remove $(u,\contextU)$ from $\mathcal{A}$
\STATE \TAB\TAB activate all pairs (child($u$), child($\contextU$))

\end{algorithmic}
\end{algorithm}

Let us assume that documents and contexts are leaves in a document tree $\DocTree$ and context tree $\ContextTree$, respectively. The algorithm (see Algorithm~\ref{alg:contextualzooming} for pseudocode) maintains a set $\mathcal{A}$ of \emph{active strategies} of the form $(u,\contextU)$, where $u$ is a subtree in $\DocTree$ and $\contextU$ is a subtree in $\ContextTree$. At any given time the active strategies partition \hlc{4}{$\Xpairs$}.
In each round, a context $h$ arrives, and one of the active strategies $(u,\contextU)$ with $h\in \contextU$ is chosen: namely the one with the maximal \emph{index}, and then a document $x\in u$ is picked uniformly at random. The index of $(u,\contextU)$ is, essentially, the best available upper confidence bound on expected rewards from choosing a document $x\in u$ given a context $h\in \contextU$. The index is defined via sample average, confidence radius~\refeq{eq:conf-rad}, and ``width" $\size(u\times \contextU)$. The latter can be any upper bound on the diameter of the product set $u\times \contextU$ in the DC-space:
\begin{align}\label{eq:context-size}
\size(u,\contextU) \geq
	\max_{x,x'\in u,\; h,h'\in \contextU} \D(x,x') + \DC(h,h').
\end{align}
The (de)activation rule ensures that the active strategies form a finer partition in the regions of the DC-space that correspond to higher rewards and more frequently occurring contexts.

\xhdr{Provable guarantees.} The provable guarantees for the contextual MAB problem are in terms of \emph{contextual regret}, which is regret is with respect to a much stronger benchmark: the best arm in hindsight \emph{for every given context}.

Regret guarantees for \contextualZooming{} focus on the \emph{DC-space} $(\Xpairs,\Dpairs)$. A very pessimistic regret bound is \eqref{eq:regret-dim} with $d = \CovDim(\Xpairs,\Dpairs)$. However, as for the zooming algorithm, much better instance-dependent bounds are possible. See Appendix~\ref{app:instance-dependent} for further discussion.

\section{New approach: ranked contextual bandits}
\label{sec:ranked-contextual}

We now present a new approach in which the upper slot selections are taken into account as a \emph{context} in the contextual MAB setting.

The slot algorithms in the \rankedBandit{} setting can make their selections sequentially. Then \hlc{1}{without loss of generality} each slot algorithm $\A_i$ knows the set $S$ of documents in the upper slots. We propose to treat $S$ as a ``context" to $\A_i$. Specifically, $\A_i$ will assume that none of the documents in $S$ is clicked, i.e.~event $Z_S$ happens (else the $i$-th slot is ignored by the user). For each such round, the \hlc{1}{click probabilities} for $\A_i$ are given by $\mu(\cdot\,| Z_S)$, which is an L-continuous function on $(X,\D)$.

\subsection{$\MCzooming$: ``light-weight'' ranked contextual algorithm}
\label{sec:ranked-contextual-light}

We first propose a simple modification to $\zooming$, called $\MCzooming$, which uses the contexts as discussed above.

Recall that in the zooming algorithm, the index of an active subtree $u$ is defined so that, assuming stochastic rewards, it is an upper confidence bound on the click probability of any document $x$ in this subtree:
\begin{align}\label{eq:index-UCB}
\mathtt{w.h.p.~~} \mathtt{index}(u) \geq \max_{x\in u}\; \mu(x).
\end{align}
Moreover, it follows from the analysis in~\citep{LipschitzMAB-stoc08} that performance of the algorithm improves if the index is decreased as long as~\eqref{eq:index-UCB} holds.

Now consider $\zooming$, and let $\A_i$ be the instance of the zooming algorithm in slot $i\geq 2$. While for $\A_i$ the rewards are no longer stochastic, our intuition for why $\zooming$ may be a good algorithm is still based on~\eqref{eq:index-UCB}. In other words, we \emph{wish} that for each context $S\subset X$ we have
\begin{align}\label{eq:index-UCB-wish}
\mathtt{w.h.p.~~} \mathtt{index}(u) \geq \max_{x\in u}\; \mu(x|Z_S),
\end{align}
and our intuition is that it is desirable to decrease the index as long as~\eqref{eq:index-UCB-wish} holds.

We will derive an upper bound on $\max_{x\in u}\; \mu(x|Z_S)$ using correlation between $u$ and $S$, and we will cap the index of $u$ at this quantity. Since
	$\mu(y|Z_S) = 0$
for any $y\in S$, we have
\begin{align}
\mu(x|Z_S) &= |\mu(x|Z_S)- \mu(y|Z_S)| \leq \D(x,y),
	\quad \forall y\in S \nonumber \\
\mu(x|Z_S) &\leq \D(x,S)
		\triangleq \textstyle{\min_{y\in S}}\; \D(x,y).
		\label{eq:cap-index}
\end{align}
In other words, if document $x$ is close to some document in $S$, the event $Z_S$ limits the conditional probability $\mu(x|Z_S)$. Therefore we can cap the index of $u$ at
    $\max_{x\in u}\; \D(x,S)$:
\begin{align*}
\mathtt{index}(u)
    \leftarrow \min\left( \mathtt{index}(u), \;\;\max_{x\in u}\; \D(x,S) \right).
\end{align*}
The version of $\zooming$ with the above ``correlation rule'' will be called $\MCzooming$.

To simplify the computation of $\max_{x\in u}\; \D(x,S)$ in an $\eps$-exponential tree metric, we note that it is equal to
    $\D(\mathtt{root}(u),\,S)$
if $u$ is disjoint with $S$, and in general it is equal to
    $\D(\mathtt{root}(v),\,S)$,
where $v$ is the largest subtree of $u$ that is disjoint with $S$.

\subsection{Contextual \emph{Lipschitz} MAB interpretation}

Let us cast each slot algorithm $\A_i$ as a contextual algorithm in the contextual \emph{Lipschitz} MAB setting (as defined in Section~\ref{sec:algorithms-prior-contextual}). We need to specify a metric $\DC$ on contexts $S\subset X$ which can be computed by the algorithm and satisfies the Lipschitz condition:
\begin{align}\label{eq:Lip-contextual-2}
|\mu(x|Z_S) - \mu(x|Z_{S'})| \leq \DC(S,S') \quad
    \text{for all } x\in X \text{ and }  S,S'\subset X.
\end{align}

\begin{lemma}\label{lm:Lip-cont-contextual}
Consider the \problem. For any $S,S'\subset X$, define
\begin{align}\label{eq:context-distance}
\DC(S,S')
	\triangleq 4\,\inf \textstyle{\sum_{j=1}^n} \D(x_j,x'_j),
\end{align}
where the infimum is taken over all $n\in\N$ and over all $n$-element sequences $\{x_j\}$ and $\{x'_j\}$ that enumerate, possibly with repetitions, all documents in $S$ and $S'$.
Then $\DC$ satisfies~\eqref{eq:Lip-contextual-2}.
\end{lemma}

\begin{proof}
For shorthand, let us write
\begin{align*}
	\sigma(x|S) &\triangleq  1-\mu(x|Z_S), \\
	\sigma(x|S,y) &\triangleq \sigma(x|S \cup \{y\}).
\end{align*}
First, we claim that for any $y\in X$ and $y'\in S$
\begin{align}\label{eq:Lip-cont-contextual-sp}
	|\sigma(x|S,y) - \sigma(x|S,y')| \leq 4\, \D(y,y').
\end{align}
Indeed, noting that
	$\sigma(x|S,y) = \sigma(y|S,x)\; \tfrac{\sigma(x|S)}{\sigma(y|S)}$,
we can re-write the left-hand side of~\eqref{eq:Lip-cont-contextual-sp} as
\begin{align}
\text{LHS}\refeq{eq:Lip-cont-contextual-sp}
	&=\sigma(x,S) \left|
		\frac{\sigma(y|S,x)}{\sigma(y|S)}
			- \frac{\sigma(y'|S,x)}{\sigma(y'|S)} \right| \nonumber \\
	& \leq \sigma(x,S) \; \D(y,y')\;
			\frac{\sigma(y|S) + \sigma(y|S,x)}
				{\sigma(y|S)\, \sigma(y'|S)}
		\label{eq:Lip-cont-contextual-eq1}\\
	& = \D(y,y')\; \frac{\sigma(x|S) + \sigma(x|S,y)}{\sigma(y'|S)}
	\leq 2\, \D(y,y'). \nonumber
\end{align}
In~\eqref{eq:Lip-cont-contextual-eq1}, we have used the L-continuity of $\sigma(\cdot|S)$ and  $\sigma(\cdot|S,x)$. To achieve the constant of $2$, it was crucial that $y'\in S$, so that $\sigma(y'|S) = 1$. This completes the proof of~\eqref{eq:Lip-cont-contextual-sp}.

Fix some $n\in\N$ and some $n$-element sequences $\{x_i\}$ and $\{x'_i\}$ that enumerate, possibly with repetitions, all values in $S$ and $S'$, respectively. Consider sets
$$S_i = \{ x'_1 \LDOTS x'_i \} \cup \{ x_{i+1} \LDOTS x_n \},\;\;
    1\leq i \leq n-1,$$
and let $S_0 = S$ and $S_{n+1} = S'$. To prove the lemma, it suffices to show that
\begin{align}\label{eq:Lip-cont-contextual-eq2}
|\sigma(x|S_i) - \sigma(x|S_{i+1}) | \leq 4\, \D(x_{i+1},\, x'_{i+1})
\end{align}
for each $i\leq n$. To prove~\eqref{eq:Lip-cont-contextual-eq2}, fix $i$ and let
	$y = x_{i+1}$ and $y' = x'_{i+1}$.
Note that
	$S_i \cup \{y'\} = S_{i+1} \cup \{y\}$,
call this set $S^*$. Then using~\eqref{eq:Lip-cont-contextual-sp} (note, $y\in S_i$ and $y'\in S'_i$) we obtain
\begin{align*}
|\sigma(x|S_i) - \sigma(x|S^*) |
	&= |\sigma(x|S_i,y) - \sigma(x|S_i,y') | \\
	&\leq  2\,\D(y,y'), \\
|\sigma(x|S_{i+1}) - \sigma(x|S^*) |
	&= |\sigma(x|S_{i+1},y') - \sigma(x|S_{i+1},y) |  \\
	&\leq 2\,\D(y,y'),
\end{align*}
which implies~\eqref{eq:Lip-cont-contextual-eq2}.
\end{proof}

\subsection{$\rankedContextualZooming$: ``full-blown'' ranked contextual algorithm}
\label{sec:ranked-contextual-heavy}

Now we can take any algorithm for the contextual Lipschitz MAB problem (with metric $\DC$ on contexts given by~\eqref{eq:context-distance}), and use it as a slot algorithm. We will use $\contextualZooming$, augmented by the ``correlation rule'' similar to the one in Section~\ref{sec:ranked-contextual-light}. The resulting ``ranked" algorithm will be called $\rankedContextualZooming$.

The implementation details are not difficult. Suppose the metric space on documents is the $\eps$-exponential tree metric, and let $\DocTree$ be the document tree. Consider slot $(i+1)$-th slot, $i\geq 1$.%
\footnote{For slot $1$, contexts are empty, so $\contextualZooming$ reduces to Algorithm~\ref{alg:zooming}.} Then the contexts are unordered $i$-tuples of documents. Let us define \emph{context tree} $\ContextTree$ as follows.
Depth-$\ell$ nodes of $\ContextTree$ are unordered $i$-tuples of depth-$\ell$ nodes from $\DocTree$, and leaves are contexts. The root of $\ContextTree$ is $(r \ldots r)$, where
	$r= \texttt{root}(\DocTree)$.
For each internal node
	$\contextU = (u_1 \ldots u_i)$
of $\ContextTree$, its children are all unordered tuples $(v_1 \ldots v_i)$ such that each $v_j$ is a child of $u_j$ in $\DocTree$. This completes the definition of $\ContextTree$. Letting $u$ and $\contextU$ be level-$\ell$ subtrees of $\DocTree$ and $\ContextTree$, respectively, it follows from the definition of $\DC$ in~\eqref{eq:context-distance} that
	$\DC(S,S') \leq 4i\, \eps^\ell$
for any contexts $S,S'\in \contextU$.
Thus setting
	$\size(u\times \contextU) \triangleq \eps^\ell (4i+1)$
satisfies~\eqref{eq:context-size}.

 We define the ``correlation rule'' as follows. Let $(u,\contextU)$ be an active strategy in the execution of \contextualZooming, where $u$ is a subtree of the document tree $\DocTree$, and $\contextU$ is a subtree of the context tree $\ContextTree$. It follows from the analysis in~\citep{contextualMAB-slivkins09} that decreasing the index of $(u,\contextU)$ improves performance, as long it holds that
\begin{align*}
\mathtt{index}(u,\contextU) \geq \mu(x|Z_S),
	\quad \forall x \in u, \, S\in \contextU.
\end{align*}
Recall that $\mu(x|Z_S) \leq \D(x,S)$ by \eqref{eq:cap-index}, so we can cap
    $\mathtt{index}(u,\contextU)$
at
    $\max_{x\in u}\; \D(x|S)$:
\begin{align*}
\mathtt{index}(u,S)
    \leftarrow \min\left( \mathtt{index}(u,S), \;\;\max_{x\in u}\; \D(x|S) \right).
\end{align*}
This completes the description of $\rankedContextualZooming$.

%% file: sec5-provable.tex
\section{Provable scalability guarantees and discussion}
\label{sec:provable}

Noting that for each slot $i\geq k$ the covering dimension of the DC-space is at most $k$ times the covering dimension of $(X,\D)$, it follows that a (very pessimistic) upper bound on contextual regret of $\rankedContextualZooming$ is
    $R(T) = \tilde{O}( \alpha \, T^{1-1/(kd+2)})$.
Plugging this into~\eqref{eq:regret-RBA}, we obtain:

\begin{theorem}\label{thm:regret-RankContZoom}
Consider the \problem{} on a metric space with covering dimension $d$ (as defined in~\eqref{eq:dim-defn}, with constant $\alpha$). Then after $T$ rounds algorithm \rankedContextualZooming{} achieves
\begin{align*}
	\frac{\E[\mathtt{\#clicks}]}{T} \geq (1-\tfrac{1}{e})\,\OPT -
            \tilde{O}\left( \frac{\alpha k}{T^{1/(kd+2)}} \right).
\end{align*}
\end{theorem}

This is just a basic scalability guarantee which does not degenerate with the number of documents. (Note that it is \emph{worse} than the one for \rankedGridEXP.) We believe that this guarantee is very pessimistic, as it builds on a very pessimistic version of the result for \contextualZooming. In particular, we ignore the intuition that for a given slot, contexts $S\subset X$ may gradually converge over time to the greedy optimum, which effectively results in a much smaller set of possible contexts.\footnote{It is also wasteful (but perhaps less so) that we use a slot-$k$ bound for each slot $i<k$.} We believe this effect is very important to the performance \rankedContextualZooming. In particular, it causes \rankedContextualZooming{} to perform much better than \rankedGridEXP{} in simulations.

\subsection{A better benchmark}
\label{sec:provable-benchmark}

Recall that while the bound in~\eqref{eq:regret-RBA} uses
	$(1-\tfrac{1}{e})\, \text{OPT}$
as a benchmark, a more natural benchmark would be the greedy optimum. We provide a preliminary convergence result for \rankedContextualZooming, without any specific regret bounds.

Such result is more elegantly formulated in terms of a version of \rankedContextualZooming, henceforth called anytime-\rankedContextualZooming, which uses the anytime version of \contextualZooming{} (see Section~\ref{sec:algs-anytime}).

\begin{theorem}\label{thm:convergence}
Fix an instance of the $k$-slot MAB problem. The performance of anytime-\rankedContextualZooming{} up to any given time $t$ is equal to the greedy optimum minus $f(t)$ such that $f(t)\to 0$.
\end{theorem}

\OMIT{Assuming a unique greedy ranking, this convergence result extends to any \rankedBandit{} such that {\tt Bandit} has sublinear regret in the stochastic MAB setting.}

\begin{proof}[Proof Sketch]
\hlc{5}{
It suffices to prove that with high probability, anytime-\rankedContextualZooming{} outputs a greedy ranking in all but $f_k(t)$ rounds among the first $t$ rounds, where $f_k(t)\to 0$.

We prove this claim by induction on $k$, the number of slots. Suppose it holds for some $k-1$ slots, and focus on the $k$-th slot. Consider all rounds in which a greedy ranking is chosen for the upper slots but not for the $k$-th slot. In each such round, the $k$-th slot replica of anytime-\contextualZooming{} incurs contextual regret at least $\delta_k$, for some instance-specific constant $\delta_k>0$. Thus, with high probability there can be at most $R_k(t)/\delta_k$ such rounds, where $R_k(t)=o(t)$ is an upper bound on contextual regret for slot $k$. Thus, one can take
    $f_k(t) = f_{k-1}(t) + R_k(t)/\delta_k$.
}
\end{proof}

\OMIT{Since \#documents is finite, by induction on $k$ we can prove that in all but $o(T)$ rounds we have a greedy ranking.}

Theorem~\ref{thm:convergence} is about the ``metric-less" setting from~\cite{RBA-icml08}. It easily extends to the ``ranked" version of any bandit algorithm whose contextual regret is sublinear with high probability.

It is an open question whether (and under which assumptions) Theorem~\ref{thm:convergence} can be extended to the ``ranked" versions of non-contextual bandit algorithms such as \naiveUCB. One assumption that appears essential is the uniqueness of the greedy ranking. To see that multiple greedy rankings may cause problems for ranked non-contextual algorithms, consider a simple example:
\begin{itemize}
\item There are two slots and three documents $x_1, x_2, x_3$ such that	
	$\mu = (\tfrac12, \tfrac12, \tfrac13)$
and the relevance of each arm is independent of that of the other arms.\footnote{Here documents $x_1, x_2, x_3$ can stand for disjoint \emph{subsets} of documents with highly correlated payoffs. Documents within a given subset can lie far from one another in the metric space.}
\end{itemize}
An optimal ranking for this example is a greedy ranking that puts $x_1$ and $x_2$ in the two slots, achieving aggregate click probability $\tfrac34$. According to our intuition, a ``reasonable" ranked non-contextual algorithm will behave as follows. The slot $1$ algorithm will alternate between $x_1$ and $x_2$, each with frequency $\to\tfrac12$. Since the slot-$2$ algorithm is oblivious to the slot $1$ selection, it will observe averages that converge over time to
	$(\tfrac14, \tfrac14, \tfrac13)$,\footnote{\hlc{5}{Suppose $x_j$, $j\in \{1,2\}$ is chosen in slot $1$. Then, letting $S=\{x_j\}$, $\mu(x_1|Z_S)$ equals $0$ if $j=1$ and $\tfrac12$ otherwise (which averages to $\tfrac14$), whereas
        $\mu(x_3|Z_S) = \tfrac13$.}}
so it will select document $x_3$ with frequency $\to 1$. Therefore frequency $\to 1$ the ranked algorithm will alternate between $(x,z)$ or $(y,z)$, each of which has aggregate click probability $\tfrac23$.

\subsection{Desiderata}

We believe that the above guarantees do not reflect the full power of our algorithms, and more generally the full power of conditional L-continuity. The ``ideal" performance guarantee for \rankedBandit{} in our setting would use the greedy optimum as a benchmark, and would have a bound on regret that is free from the inefficiencies outlined in the discussion after Theorem~\ref{thm:regret-RankContZoom}. Furthermore, this guarantee would only rely on some general property of $\mathtt{Bandit}$ such as a bound on regret or contextual regret. We conjecture that such guarantee is possible for \rankedContextualZooming, and, perhaps under some assumptions, also for \MCzooming{} and \zooming.

Further, one would like to study the relative benefits of the new ``contextual" algorithms (\rankedContextualZooming{} and \MCzooming) and the prior work such as $\zooming$. The discussion Section~\ref{sec:provable-benchmark} suggests that the difference can be particularly pronounced when the pointwise mean has multiple peaks of similar value. In fact, we confirm this experimentally in Section~\ref{sec:evaluation-secondary}.

%% file: sec6-eval.tex
\section{Evaluation}
\label{sec:evaluation}

Let us evaluate the performance of the algorithms presented in Section \ref{sec:algorithms-prior} and Section~\ref{sec:ranked-contextual}. We summarize these algorithms in Table~\ref{tab:algs-prior}.

\begin{table}[t]
\caption{Algorithms for the \problem.}
\label{tab:algs-prior}
\vspace{3mm}

\begin{center}
\begin{tabular}{c|l|l}
$\naiveUCB$         & metric-oblivious algorithms:
                    & Section~\ref{sec:prior-RBA} \\
$\naiveEXP$         & ~~~~ranked versions of $\UCB$ and $\EXP$
                    & \\ \hline
$\rankedGridUCB$    & simple metric-aware algorithms:
                    & Section~\ref{sec:prior-rankedMetric}\\
$\rankedGridEXP$    & ~~~~ranked versions of $\gridUCB$ and $\gridEXP$
                    & \\ \hline
$\zooming$          & the ranked version of the zooming algorithm
                    & Section~\ref{sec:prior-rankedMetric}\\ \hline
                    & contextual algorithms: \\
$\MCzooming$                & ~~~~``light-weight''(based on the zooming algorithm)
                            & Section~\ref{sec:ranked-contextual-light} \\
$\rankedContextualZooming$  & ~~~~``full-blown'' (based on $\contextualZooming$).
                            & Section~\ref{sec:ranked-contextual-heavy}
\end{tabular}
\end{center}
\end{table}

\OMIT{ 
 ``metric-oblivious" \naiveUCB{} and \naiveEXP, ``metric-aware" non-contextual \rankedGridUCB, \rankedGridEXP{} and \zooming, and contextual $\rankedContextualZooming$
and \MCzooming.
} 

In all \UCB-based algorithms in Table~\ref{tab:algs-prior}, including all extensions of the zooming algorithm, one can damp exploration by replacing the $4\log(T)$ factor in~\eqref{eq:conf-rad} with $1$. Such change effectively makes the algorithm more \emph{optimistic}; it was found beneficial for \naiveUCB{} by~\citet{RBA-icml08}. We find (see Section~\ref{sec:eval-optimistic}) that this change greatly improves the average performance in our experiments. So, by a slight abuse of notation, we will assume this change from now on.


\subsection{Experimental setup}

Using the generative model from Section \ref{sec:model} (Algorithm~\ref{alg:users} with~\eqref{eq:model-q}), we created a document collection with
	$|X| = 2^{15} \approx 32,000$
documents\footnote{This is a realistic number of documents that may be considered in detail for a typical web search query after pruning very unlikely documents.} in a binary \eps-exponential tree metric space with $\epsilon = 0.837$ (and constant $c=1$, see Section~\ref{sec:defns-example}). The value for $\epsilon$ was chosen so that the most dissimilar documents in the collection still have a non-trivial similarity, as may be expected for web documents. Each document's expected relevance $\mu(x)$ was set by first identifying a small number of ``peaks" $y_i \in X$, choosing $\mu(\cdot)$ for these documents, and then defining the relevance of other documents as the minimum allowed while obeying L-continuity and a background relevance \hlc{1}{rate} $\mu_0$:
\begin{align}\label{eq:simulation-mu}
	\mu(x) \triangleq \max(\mu_0,\; \tfrac12 - \textstyle{\min_i}\, \D(x,y_i)).
\end{align}
For internal nodes in the tree, $\mu$ is defined bottom-up (from leaves to the root) as the mean value of all children nodes. \hlFR{As a result, we obtain a set of documents $X$ where each document $x \in X$ has an expected click probability $\mu(x)$ that obeys L-continuity.}

Our simulation was run over a 5-slot ranked bandit setting, learning the best 5 documents. We evaluated over  300,000 user visits sampled from $\P$ per Algorithm~\ref{alg:users}. Performance within 50,000 impressions, typical for the number of times relatively frequent queries are seen by commercial search engines in a month, is essential for any practical applicability of this approach. \hlFR{However, we also measure performance for a longer time period to obtain a deeper understanding of the  convergence properties of the algorithms.}

We consider two models for $\mu(\cdot)$ in~\eqref{eq:simulation-mu}. In the first model, two ``peaks" $\{y_1, y_2\}$ are selected at random with $\mu(\cdot) = \tfrac12$, and $\mu_0$ set to 0.05. The second model is less ``rigid" (and thus more realistic): the relevant documents $y_i$ and their expected relevance rates $\mu(\cdot)$ are selected according to a Chinese Restaurant Process \citep{Aldous85Exchangeability}  \hlFR{ with parameters $n\!=\!20$ and $\theta\!=\!2$, and setting $\mu_0 = 0.01$. The Chinese Restaurant Process is inspired by customers coming in to a restaurant with an infinite number of tables, each with infinite capacity. At time $t$, a customer arrives and can choose to sit at a new table with probability $\theta / (t-1+\theta)$, and otherwise sits at an already occupied table with probability proportional to the number of customers already sitting at that table. By considering each table as equivalent to a peak in the distrubtion, this leads to a set of peaks with expected relevance rates distributed accoring to a power law. Following~\mbox{\cite{RBA-icml08}}, we assign users to one of the peaks, then select relevant documents so as to obey the expected relevance rate  $\mu(x)$ for each document $x$.}

As baselines we use an algorithm ranking the documents at random, and the (offline) greedy algorithm discussed in Section~\ref{sec:prior-RBA}.

\OMIT{another ranking them according in an offline greedy manner (see Section~\mbox{\ref{sec:provable-prior}} for a discussion of the offline greeding ranking).}

\begin{figure*}[p]
\centering{\includegraphics[width=15cm,height=9cm]{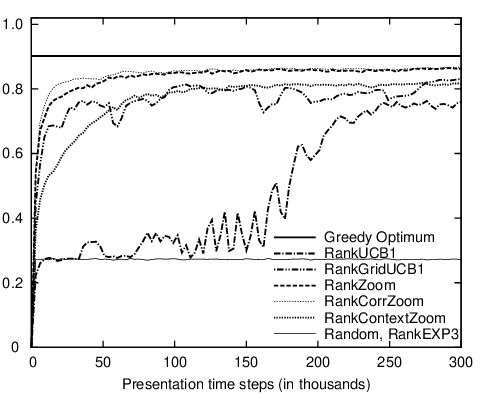}}
\caption{The learning algorithms on 5-slot problem instances with two relevance peaks.}
\label{fig1a}
\end{figure*}

\begin{figure*}[p]
\centering{\includegraphics[width=15cm,height=9cm]{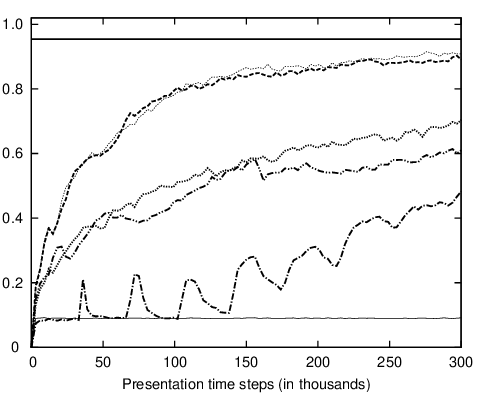}}
\caption{The learning algorithms on 5-slot problem instances with random relevance rates $\mu(\cdot)$ selected according to the Chinese Restaurant Process.}
\label{fig1b}
\end{figure*}

\subsection{Main experimental results}

Our experimental results are summarized in Figure~\ref{fig1a} and Figure~\ref{fig1b}.

\naiveEXP{} and \naiveUCB{} perform as poorly as picking documents randomly: the three curves are indistinguishable. This is due to the large number of available documents and slow convergence rates of these algorithms. Other algorithms that explore all strategies (such as REC \citep{RBA-icml08}) would perform just as poorly. This result is consistent with results reported by \cite{RBA-icml08} on just  $50$ documents. \hlFR{On the other hand, algorithms that progressively refine the space of strategies explored perform much better.}

$\MCzooming$ achieves the best empirical performance, converging rapidly to near-optimal rankings. $\zooming$ is a close second. The theoretically preferred $\rankedContextualZooming$ comes third, with a significant gap. This appears to be due to the much larger branching factor in the strategies activated by $\rankedContextualZooming$ slowing down the convergence. (However, as we investigate in Section~\ref{sec:evaluation-secondary}, $\rankedContextualZooming$ may significantly outperform the other algorithms if $\mu$ has multiple peaks with similar values.)



\subsection{``Optimistic'' vs. ``pessimistic'' \UCB-style algorithms}
\label{sec:eval-optimistic}

We find that the ``optimistic'' \UCB-style algorithms (obtained by replacing the $4\log(T)$ factor in~\eqref{eq:conf-rad} with $1$) perform dramatically better than their ``pessimistic'' counterparts. In Figure~\ref{fig1a-plus} and Figure~\ref{fig1b-plus} we compare $\naiveUCB$ and $\zooming$ with their respective ``pessimistic'' versions (which are marked with a ``-~\!-'' after the algorithm name). We saw a similar increase in performance for other \UCB-style algorithms, too.

\begin{figure*}[p]
\centering{\includegraphics[width=15cm,height=8cm]{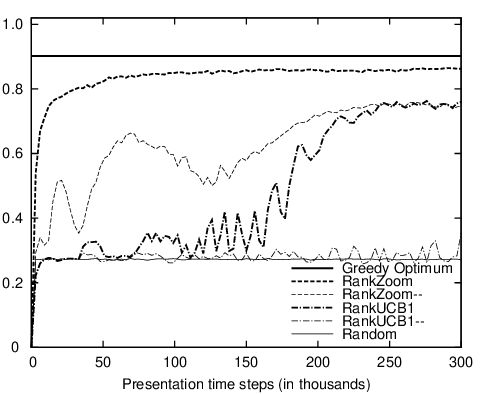}}
\caption{``Optimistic'' vs. ``pessimistic'' \UCB-style algorithms: \newline
The learning algorithms on 5-slot problem instances with two relevance peaks.}
\label{fig1a-plus}
\end{figure*}

\begin{figure*}[p]
\centering{\includegraphics[width=15cm,height=8cm]{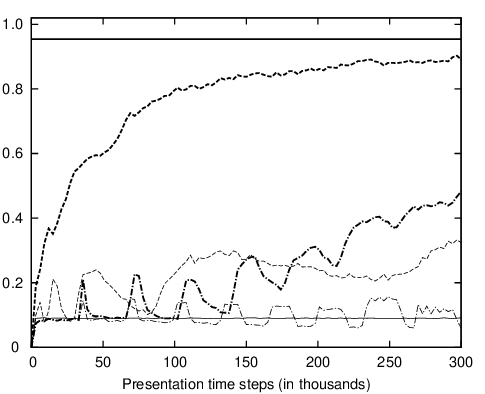}}
\caption{``Optimistic'' vs. ``pessimistic'' \UCB-style algorithms: \newline
The learning algorithms on 5-slot problem instances with random relevance rates $\mu(\cdot)$ selected according to the Chinese Restaurant Process.}
\label{fig1b-plus}
\end{figure*}

\subsection{Secondary experiment}
\label{sec:evaluation-secondary}

\begin{figure*}[t]
\centering
\includegraphics[width=15cm,height=8cm]{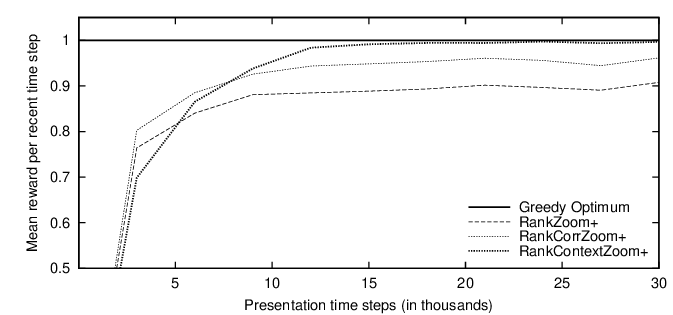}
\vspace{-3mm}
\caption{Zooming-style algorithms in a two-slot setting over a small document collection.}
\label{fig2}
\end{figure*}

As discussed in Section~\ref{sec:provable-benchmark}, some \rankedBandit-style algorithms may converge to a suboptimal ranking if $\mu$ has multiple peaks with similar values. To investigate this,
we designed a small-scale experiment presented in Figure~\ref{fig2}. We generated a small collection of 128 documents using the same setup with two ``peaks", and assumed $2$ slots. Each peak corresponds to a half of the user population, with peak value $\mu = \tfrac12$ and background value $\mu_0 = 0.05$.

We see that $\rankedContextualZooming$ converges more slowly than the other zooming variants, but eventually outperforms them. This confirms our intuition, and suggests that $\rankedContextualZooming$ may eventually outperform the other algorithms on a larger collection, such as that used for Figures 1 and 2.

\OMIT{ 
To summarize, the \MCzooming{} algorithm performs best in this setting, converging much more rapidly than
previous approaches. Correlation information improves its performance somewhat over \zooming. On the other hand, \rankedContextualZooming, while theoretically
superior converges more slowly.
} 

%% file: sec7-proofs.tex
\section{Proof of Lemma~\ref{lm:extension-mu-short} (extending $\mu$ from leaves to tree nodes)}
\label{sec:extending-mu}

Recall that Lemma~\ref{lm:extension-mu-short} is needed to define the generative model in Section~\ref{sec:model}. We will prove a slightly more general statement:

\begin{lemma}\label{lm:extension-mu}
Let $\D$ be the shortest-paths metric of an edge-weighted rooted tree with node set $V$ and leaf set $X$. Let $\mu:X\to [a,b]$ be an L-continuous function on $(X,\D)$. Then $\mu$ can be extended to $V$ so that $\mu:V \to [a,b]$ is L-continuous w.r.t. $(V,\D)$.
\end{lemma}

\begin{proof}
For each $x\in V$, let $\mathcal{L}(x)$ be the set of all leaves in the subtree rooted at $x$. For each $z\in \mathcal{L}(y)$ the assignment $\mu(x)$ should satisfy
\begin{align*}
\mu(z) - \D(x,z) \leq \mu(x) \leq \mu(z) + \D(x,z)
\end{align*}
Thus $\mu(x)$ should lie in the interval
	$I(x) \triangleq [\mu^-(x), \mu^+(x) ]$,
where
\begin{align*}
\mu^-(x) &\triangleq \textstyle{\sup_{z\in \mathcal{L}(x)}}\; \mu(z) - \D(x,z), \\
\mu^+(x) &\triangleq \textstyle{\inf_{z\in \mathcal{L}(x)}}\; \mu(z) + \D(x,z).
\end{align*}
This interval is always well-defined, i.e. $\mu^-(x)\leq \mu^+(x)$. Indeed, if not then for some
	$z,z'\in \mathcal{L}(x)$
\begin{align*}
	\mu(z)-\D(x,z) 	 &> \mu(z')+ \D(x,z') \\
	\mu(z) - \mu(z') &> \D(x,z) + \D(x,z') \geq \D(z,z'),
\end{align*}
contradiction, claim proved. Note that $\mu^+(x)\geq a$ and $\mu^-(x)\leq b$, so the intervals $I(x)$ and $[a,b]$ overlap.

Using induction on the tree, we will construct values $\mu(x)$, $x\in V$ such that the Lipschitz condition
\begin{align*}
|\mu(x)-\mu(y)| \leq \D(x,y) \qquad \text{for all $x,y\in X$}
\end{align*}
holds whenever $x$ is a parent of $y$. For the root $x_0$, let $\mu(x_0)$ be an arbitrary value in the interval $I(x_0)\cap [a,b]$. For the induction step, suppose for some $x$ we have chosen $\mu(x)\in I(x)\cap [a,b]$ and $y$ is a child of $x$. We need to choose $\mu(y)\in I(y)\cap [a,b]$ so that
	$|\mu(x)-\mu(y)| \leq \D(x,y)$.
Note that
\begin{align*}
\mu(x)
	&\geq \mu^-(x)
	\geq \textstyle{\sup_{z\in \mathcal{L}(y)}}\; \left[\, \mu(z) - \D(x,y) - \D(y,z) \,\right] \\
	&= \mu^-(y) - \D(x,y), \\
\mu(x)
	&\leq \mu^+(x)
	\leq \textstyle{\inf_{z\in \mathcal{L}(y)}}\; \left[\, \mu(z) + \D(x,y) + \D(y,z) \,\right] \\
	&= \mu^+(y) + \D(x,y).
\end{align*}
It follows that $I(y)$ and
	$[\mu(x)-\D(x,y),\, \mu(x)+\D(x,y)]$
have a non-empty intersection. Therefore, both intervals have a non-empty intersection with $[a,b]$. So we can choose $\mu(y)$ as required. This completes the construction of $\mu()$ on $V$.

To check that $\mu$ is Lipschitz-continuous on $V$, fix $x,y\in V$, let $P$ be the $x\to y$ path in the tree, and note that
\begin{align*}
|\mu(x) - \mu(y)|
	&\leq \textstyle{\sum_{(u,v)\in P}}\, |\mu(u)-\mu(v)| \\
	&\leq \textstyle{\sum_{(u,v)\in P}}\, \D(u,v)
	= \D(x,y).\qquad\qedhere
\end{align*}
\end{proof}

\section{Proof of Theorem~\ref{thm:expressiveness} (expressiveness of the model)}
\label{sec:proof-conditional}

Recall that a proof sketch for Theorem~\ref{thm:expressiveness} was given in Section~\ref{sec:model}.
In this section we complete this proof sketch by proving Equation~\refeq{eq:app-expressiveness-conditional}.

\xhdr{Notation.}
Let us introduce the notation (some of it is from the proof sketch).

For a tree node $u$, let $\mathcal{T}_u$ be the node set of the subtree rooted at $u$. For convenience (and by a slight abuse of notation) we will write $u=b$, $b\in \{0,1\}$ to mean $\pi(u) = b$.

Fix documents $x,y\in X$. We focus on the key event, denoted $\mathcal{E}$,  that no mutation happened on the $x\to y$ path. Recall that in Algorithm 1, for each tree node $u$ with parent $v$ we assign
	$\pi(u) \leftarrow M_u(\pi(v))$,
where 	
	$M_u: \{0,1\} \to \{0,1\}$
is a random mutation which flips the input bit $b$ with probability $q_b(u)$. If $M_u$ is the identity function, then we say that no mutation happened at $u$. We say that no mutation happened on the $x\to y$ path if no mutation happened at each node in $N_{xy}$, the set of all nodes on the $x\to y$ path except $z$. This event is denoted $\mathcal{E}$; note that it implies
	$\pi(x) = \pi(y) = \pi(z)$.
Its complement $\bar{\mathcal{E}}$ is, intuitively, a low-probability ``failure event".

Fix a subset of documents $S\subset X$. Recall that $Z_S$ denotes the event that all documents in $S$ are irrelevant, i.e. $\pi(x) = 0$ for all $x\in S$. 

\xhdr{What we need to prove.} We need to prove~\eqref{eq:app-expressiveness-conditional}, which states that
\begin{align*}
\Pr[\bar{\mathcal{E}} \,|\, Z_S] \leq 3\, \Pr[\bar{\mathcal{E}}].
\end{align*}

It suffices to prove the following lemma:

\begin{lemma}\label{lm:app-expressiveness-conditional}
$ \Pr[\bar{\mathcal{E}} \,|\, Z_S] \leq
		\Pr[\bar{\mathcal{E}}] \times (2/\Pr[\mathcal{E}])
$.
\end{lemma}

\noindent (Indeed, letting $p=\Pr[\bar{\mathcal{E}}]$ it holds that
$\Pr[\bar{\mathcal{E}} \,|\, Z_S]
	\leq \min\left(1,\, \tfrac{2p}{1-p} \right)
	\leq 3\,p$. )

\begin{note}{Remark.}
Lemma~\ref{lm:app-expressiveness-conditional} inherits assumptions (\ref{eq:construction-consistence}-\ref{eq:construction-small}) on the mutation probabilities. Specifically for this Lemma, the upper bound~\refeq{eq:construction-consistence} on mutation probabilities can be replaced with a much weaker upper bound:
\begin{align}\label{eq:weak-assn}
\max(q_0(u),\, q_1(u)) \leq \tfrac12
	\quad\text{for each tree node $u$}.
\end{align}
\end{note}

Our goal is to prove Lemma~\ref{lm:app-expressiveness-conditional}. 
In a sequence on claims, we will establish that
\begin{align}\label{eq:zero-is-better}
\Pr[Z_S \,|\, z=0] \geq \Pr[Z_S \,|\, z=1].
\end{align}
Intuitively,~\refeq{eq:zero-is-better} means that the low-probability mutations are more likely to zero out a given subset of the leaves if the value at some fixed internal node is zero (rather than one).

\subsection{Using Equation~\refeq{eq:zero-is-better} to prove Lemma~\ref{lm:app-expressiveness-conditional}}

Let us extend the notion of mutation from a single node to the $x\to y$ path. Recall that $N_{xy}$ denotes the set of all nodes on this path except $z$. Then the individual node mutations $\{M_u: u\in N_{xy}\}$ collectively provide a mutation on $N_{xy}$, which we define simply as a function
	$M: N_{xy}\times \{0,1\} \to \{0,1\}$
such that
	$\pi(\cdot) = M(\cdot,\pi(z))$.
Crucially, $M$ is chosen independently of $\pi(z)$ (and of all other mutations). Let $\mathcal{M}$ be the set of all possible mutations of $N_{xy}$. By a slight abuse of notation, we treat the event $\mathcal{E}$ as the identity mutation.

\begin{claim}\label{cl:no-mutations-is-better}
Fix $M \in \mathcal{M}$ and $b\in\{0,1\}$. Then
\begin{align*}
	\Pr[Z_S \,|\, M, \pi(z)=b] \leq \Pr[Z_S \,|\, \mathcal{E}, \pi(z)=0].
\end{align*}
\end{claim}

\begin{proof}
For each tree node $u$, let $S_u = S\cap \mathcal{T}_u$ be the subset of $S$ that lies in the subtree $\mathcal{T}_u$. Then by~\refeq{eq:zero-is-better}
\begin{align*}
\Pr[Z_S \,|\, M,\,\pi(z) = b]
	&= \textstyle{\prod_u} \;
		\Pr[Z_{S_u} \,|\, \pi(u) = M(u,b) ] \\
	&\leq \textstyle{\prod_u}\;
		\Pr[Z_{S_u} \,|\, \pi(u) = 0 ] \\
	&= \Pr[Z_S \,|\, \mathcal{E},\, \pi(z) = 0],
\end{align*}
where the product is over all tree nodes $u\in N_{xy}$ such that the intersection $S_u$ is non-empty.
\end{proof}

\begin{proof}[Proof of Lemma~\ref{lm:app-expressiveness-conditional}]
On one hand, by Claim~\ref{cl:no-mutations-is-better}
\begin{align*}
\Pr[Z_S \cap \bar{\mathcal{E}}]
	&=  \textstyle{\sum_{b,M}}\;
		\Pr[M]\; \Pr[z=b]\; \Pr[Z_S \,|\, M,\, z=b] \\
	&\leq  \textstyle{\sum_{b,M}}\;
		\Pr[M]\; \Pr[z=b]\; \Pr[Z_S \,|\, \,\mathcal{E},\, z=0] \\
	&= \Pr[\bar{\mathcal{E}}] \times \Pr[Z_S \,|\, \,\mathcal{E},\, z=0],
\end{align*}
where the sums are over bits $b\in \{0,1\}$ and all mutations
	$M\in \mathcal{M} \setminus \{ \mathcal{E}\}$.
On the other hand,
\begin{align*}
\Pr[Z_S]
	&= \textstyle{\sum_{b,M}}\;
		\Pr[M]\; \Pr[z=b]\; \Pr[Z_S \,|\, M,\, z=b] \\
\intertext{(where the sum is over $b\in \{0,1\}$ and $M\in \mathcal{M}$)}
	&\geq \Pr[\mathcal{E}] \; \Pr[z=0]\;
		\Pr[Z_S \,|\, \mathcal{E},\, z=0].
\end{align*}
Since $\Pr[z=0] \geq \tfrac12$, it follows that
\begin{align*}
\Pr[\bar{\mathcal{E}} \,|\, Z_S]
	&= \Pr[Z_S \cap \bar{\mathcal{E}}]  \,/\, \Pr[Z_S] \\
	&\leq 2\,\Pr[\bar{\mathcal{E}}] / \Pr[\mathcal{E}]. \qedhere
\end{align*}
\end{proof}

\subsection{Proof of Equation~\refeq{eq:zero-is-better}}

First we prove~\refeq{eq:zero-is-better} for the case $S \subset \mathcal{T}_z$, then we build on it to prove the (similar, but considerably more technical) case $S\cap T_z = \emptyset$. The general case follows since the events $Z_{S\cap \mathcal{T}_z}$ and $Z_{S\setminus \mathcal{T}_z}$ are conditionally independent given $\pi(z)$.

\begin{claim}\label{cl:zero-is-better-subset}
 If $S \subset \mathcal{T}_z$ then~\refeq{eq:zero-is-better} holds.
\end{claim}

\begin{proof}
Let us use induction the depth of $z$. For the base case, the case $x=y=z$. Then $S = \{z\}$ is the only possibility, and the claim is trivial.

For the induction step, consider children $u_i$ of $z$ such that the intersection $S_i \triangleq S\cap \mathcal{T}_{u_i}$ is non-empty. Let $u_1 \LDOTS u_k$ be all such children. For brevity, denote
	$Z_i \triangleq Z_{S_i}$, and
\begin{align*}
\nu_i(a|b) \triangleq \Pr[u_i=a \,|\, z=b],
	\quad a,b\in\{0,1\}.
\end{align*}
Note that
	$v_i(1,0) = q_0(x_i)$ and $v_i(0,1) = q_1(x_i)$.

Then for each $b\in \{0,1\}$ we have
\begin{align}
\Pr[Z_S \,|\, z=b]
	&= \textstyle{\prod_{i=1}^k}\;
		\Pr[Z_i \,|\, z=b]
	\label{cl:zero-is-better-subset-eq1}\\
\Pr[Z_i \,|\, z=b]
	&= \textstyle{\sum_{a\in \{0,1\}}}\;
		\nu_i(a|b) \; \Pr[Z_i \,|\, u_i=a].
	\label{cl:zero-is-better-subset-eq2}
\end{align}
By~\refeq{cl:zero-is-better-subset-eq1}, to prove the claim it suffices to show that
\begin{align} \label{cl:zero-is-better-subset-eq3}
\Pr[Z_i \,|\, z=0] \geq \Pr[Z_i \,|\, z=1]
\end{align}
holds for each $i$. By the induction hypothesis we have
\begin{align}\label{cl:zero-is-better-subset-eq4}
	\Pr[Z_i \,|\, u_i=0] \geq \Pr[Z_i \,|\, u_i=1].
\end{align}
Combining~\refeq{cl:zero-is-better-subset-eq4} and~\refeq{eq:weak-assn}, and noting that by~\refeq{cl:zero-is-better-subset-eq2} we have
	$\nu_i(0|0)\geq \nu_i(0|1)$,
it follows that
\begin{align*}
\Pr[Z_i \,|\, z=0] - \Pr[Z_i \,|\, z=1] &  \\
	&\hspace{-35mm}= \textstyle{\sum_{a\in \{0,1\}}}\;
		\Pr[Z_i \,|\, u_i=a] \;
			\left(\, \nu_i(a|0) - \nu_i(a|1) \,\right) \\
	&\hspace{-35mm} \geq
		\Pr[Z_i \,|\, u_i=1] \quad
			\textstyle{\sum_{a\in \{0,1\}}}\;
			\left(\, \nu_i(a|0) - \nu_i(a|1) \,\right) \\
	&\hspace{-35mm}= 0
\end{align*}
because
	$\nu_i(0|0) + \nu_i(1|0) = \nu_i(0|1) + \nu_i(1|1) = 1 $.
\end{proof}

\begin{corollary}\label{cor:zero-is-better-singleton}
Consider tree nodes $r,v,w$ such that $r$ is an ancestor of $v$ which in turn is an ancestor of $w$. Then for any $c\in\{0,1\}$
\begin{align*}
\Pr[u=0 \,|\, w=0,\, r=c] \geq \Pr[u=0 \,|\, w=1,\, r=c].
\end{align*}
\end{corollary}

\begin{proof}
We claim that for each $b\in\{0,1\}$
\begin{align}\label{eq:zero-is-better-singleton}
\Pr[w=b \,|\, u=b] \geq \Pr[w=b \,|\, u=1-b].
\end{align}
Indeed, truncating the subtree $\mathcal{T}_w$ to a single node $w$ and
specializing Lemma~\ref{cl:zero-is-better-subset} to a singleton set $S = \{w\}$ (with $z=u$) we obtain~\refeq{eq:zero-is-better-singleton} for $b=0$. The case $b=1$ is symmetric.

Now, for brevity we will omit conditioning on $\{r=c\}$ in the remainder of the proof. (Formally, we will work on in the probability space obtained by conditioning on this event.) Then for each $b\in \{0,1\}$
\begin{align*}
&\Pr[u=0 \,|\, w=b] \\
&\qquad = \frac{\Pr[u=0 \wedge w=b]}
		    {\Pr[u=0 \wedge w=b] \cup \Pr[u=1 \wedge w=b]} \\
&\qquad = \frac{1}{1+\Phi(b)},
\end{align*}
where
\begin{align*}
\Phi(b)
	&\triangleq \frac{\Pr[u=1 \wedge w=b]}{\Pr[u=0 \wedge w=b]} \\
	& = \frac{\Pr[w=b \,|\, u=1]\, \Pr[u=1] }
		     {\Pr[w=b \,|\, u=0]\, \Pr[u=0] }
\end{align*}
is decreasing in $b$ by~\refeq{eq:zero-is-better-singleton}.
\end{proof}

We will also need a stronger, \emph{conditional}, version of Lemma~\ref{cl:zero-is-better-subset} whose proof is essentially identical (and omitted).

\begin{claim}\label{cl:zero-is-better-cond}
Suppose $S \subset \mathcal{T}_z$ and $u\neq z$ is a tree node such that $\mathcal{T}_u$ is disjoint with $S$. Then
\begin{align}
\Pr[Z_S \,|\, z=0,\, u=1] \geq \Pr[Z_S \,|\, z=1,\, u=1].
\end{align}
\end{claim}

We will use Corollary~\ref{cor:zero-is-better-singleton} and Lemma~\ref{cl:zero-is-better-cond} to prove~\refeq{eq:zero-is-better} for the case $S\cap T_z = \emptyset$.

\begin{claim}\label{cl:zero-is-better-disj}
If $S$ is disjoint with $\mathcal{T}_z$ then~\refeq{eq:zero-is-better} holds.
\end{claim}

\begin{proof}
Suppose $S$ is disjoint with $\mathcal{T}_z$, and let $r$ be the root of the tree. We will use induction on the tree to prove the following: for each $c\in \{0,1\}$,
\begin{align}\label{cl:zero-is-better-disj-hyp}
\Pr[Z_S \,|\, r = c,\, z=0] \geq \Pr[Z_S \,|\, r = c,\, z=1]
\end{align}
For the induction base, consider a tree of depth 2, consisting of the root $r$ and the leaves. Then $z\not\in S$ is a leaf, so $Z_S$ is independent of $\pi(z)$ given $\pi(r)$, so~\refeq{cl:zero-is-better-disj-hyp} holds with equality.

For the induction step, fix $c\in \{0,1\}$. Let us set up the notation similarly to the proof of Claim~\ref{cl:zero-is-better-subset}. Consider children $u_i$ of $r$ such that the intersection $S_i \triangleq S\cap \mathcal{T}_{u_i}$ is non-empty. Let $u_1 \LDOTS u_k$ be all such children. Assume $z\in \mathcal{T}_{u_i}$ for some $i$ (else, $Z_S$ is independent from $\pi(z)$ given $\pi(r)$, so~\refeq{cl:zero-is-better-disj-hyp} holds with equality); without loss of generality, assume this happens for $i=1$. For brevity, for $a,b \in\{0,1\}$ denote
\begin{align*}
f_i(a,b) 	&\triangleq \Pr[Z_{S_i} \;\;\;\;\;|\, u_i = a,\, z=b] \\
\nu_i(a|b)	&\triangleq \Pr[u_i = a \,|\, r=c,\,\;\; z=b].
\end{align*}
Note that $f_i(a,b)$ and $\nu_i(a|b)$ do not depend on $b$ for $i>1$.

Then for each $b\in \{0,1\}$
\begin{align*}
&\Pr[Z_S \,|\, r = c,\, z=b] \\
&\qquad = \sum_{a_i\in \{0,1\},\; i\geq 1} \;\;
		\prod_{i\geq 1}\; f_i(a_i,b)\; \nu_i(a_i|b)\\
&\qquad =  \Phi \times \textstyle{\sum_{a\in \{0,1\}}}\;
		f_1(a,b)\; \nu_1(a|b),
\end{align*}
where
\begin{align*}
\Phi \triangleq
	\sum_{a_i\in \{0,1\},\; i\geq 2} \;\;
		\prod_{i\geq 2}\; f_i(a_i,b)\, \nu_i(a_i|b)
\end{align*}
does not depend on of $b$. Therefore:
\begin{align}
&\Pr[Z_S \,|\, r = c,\, z=1] - \Pr[Z_S \,|\, r = c,\, z=1]
		\nonumber \\
& \quad = \Phi \times \textstyle{\sum_{a\in \{0,1\}}} \nonumber\\
	& \qquad\qquad \quad
		\left[\; f_1(a,0)\, \nu_1(a|0) - f_1(a,1)\, \nu_1(a|1) \;\right]
		\label{eq:zero-is-better-disj-eq0}\\
& \quad \geq \Phi \times \textstyle{\sum_{a\in \{0,1\}}}\;
		f_1(a,1) \,\left[\; \nu_1(a|0) - \nu_1(a|1) \;\right]
		\label{eq:zero-is-better-disj-eq1}\\ 	
& \quad \geq \Phi \times  f_1(1,1) \;
		\textstyle{\sum_{a\in \{0,1\}}}\;
			\left[\; \nu_1(a|0) - \nu_1(a|1) \;\right]
		\label{eq:zero-is-better-disj-eq2}\\		
& \quad =0.
		\label{eq:zero-is-better-disj-eq3}
\end{align}
The above transitions hold for the following reasons:
\begin{description}
\item[($\ref{eq:zero-is-better-disj-eq0}\!\to\!\ref{eq:zero-is-better-disj-eq1}$)]\hspace{-2.5mm}
By Induction Hypothesis,
	$f_1(a,0) \geq f_1(a,1)$

\item[($\ref{eq:zero-is-better-disj-eq1}\!\to\!\ref{eq:zero-is-better-disj-eq2}$)]\hspace{-2.5mm}
By Lemma~\ref{cl:zero-is-better-cond}
	$f_1(0,1) \geq f_1(1,1)$,
and moreover we have $\nu_1(0|0) \geq \nu_1(0|1)$ by Corollary~\ref{cor:zero-is-better-singleton}.

\item[($\ref{eq:zero-is-better-disj-eq2}\!\to\!\ref{eq:zero-is-better-disj-eq3}$)]\hspace{-2mm}
	Since $\nu_i(0|0) + \nu_i(1|0) = \nu_i(0|1) + \nu_i(1|1) = 1 $
\end{description}
This completes the proof of the inductive step.
\end{proof}

\OMIT{ 
Here the transition
$(\ref{eq:zero-is-better-disj-eq0}\to\ref{eq:zero-is-better-disj-eq1})$ holds because by Induction Hypothesis,
	$f_1(a,0) \geq f_1(a,1)$.
The transition
$(\ref{eq:zero-is-better-disj-eq1x}\to\ref{eq:zero-is-better-disj-eq2})$ holds because by Lemma~\ref{cl:zero-is-better-cond}
	$f_1(0,1) \geq f_1(1,1)$,
and $\nu_1(0|0) \geq \nu_1(0,1)$.
Finally, the transition
$(\ref{eq:zero-is-better-disj-eq2}\to\ref{eq:zero-is-better-disj-eq3})$ holds because
	$\nu_i(0|0) + \nu_i(1|0) = \nu_i(0|1) + \nu_i(1|1) = 1 $.
} 

\OMIT{ 
Our goal is to define a distribution $\P$ over functions $\pi:X\to \{0,1\}$. If $x\in X$ is a proper leaf, just set $\pi(x) = \pi_T(x)$. Else, $x$ is an infinite path $x = (x_0,x_1,x_2, \ldots)$ away from the root $x_0$. Consider the random events
	$\mathcal{E}(x_i)$, $i\in \N$.
By construction,
\begin{align}\label{eq:thm-existence-Ey}
\textstyle{\sum_{i\in\N} \Pr[\mathcal{E}(x_i)]}
	&\leq \textstyle{\sum_{i\in\N} \D(x_i, x_{i-1})}/ \alpha
	= \D(x_0,x)/\alpha < \infty.
\end{align}
Thus by Borel-Cantelli Lemma, with probability $1$ there exists an $i_0$ such that for any $i\geq i_0$ none of the events $\mathcal{E}(x_i)$ hold. It follows that for each $i\geq i_0$ we have
	$\pi(x_i) = \pi(x_{i_0})$.
Define
	$\pi(x) = \pi(x_{i_0})$.
This completes our construction of $\pi()$.

Note that~\refeq{eq:thm-existence-Ey} is the only place in the proof where we need $\alpha>0$. In particular, this condition is not used if the tree is finite (and hence all leaves are proper).
} 

\OMIT{ 
It follows from~\refeq{eq:construction-consistence} that the pointwise mean of $\P$ is indeed $\mu$: for each node $y\in V$ with parent $x$,
\begin{align*}
\E[\pi(y)]
	&= \E[\zeta(y, \pi(x))]
	= \E[ (1-\pi(x))\, q_0(y) + \pi(x)(1-q_1(y))] \\
	& = (1-\mu(x))\, q_0+ \mu(x)(1-q_1) = \mu(y).
\end{align*}
}

\OMIT{ 
For a non-proper leaf $x = (x_0,x_1,x_2, \ldots)$, recall that
	$\Pr[\pi(x) = \lim_{i\to \infty} \pi(x_i)]=1$. Therefore
\begin{align*}
\E[\pi(x)]
	= \lim_{i\to\infty} \E[\pi(x_i)]
	= \lim_{i\to\infty} \mu(x_i)
	= \mu(x).
\end{align*}
} 

\OMIT{ 
{\bf From Section 4.2.} In Section 4.2, we note that in the {\tt RankedBandits} algorithm, the expected payoff of the slot $i$ algorithm, conditional on this algorithm being invoked, is $\mu(x|S)$, where $x$ is the document chosen, and $S$ is the set of documents in the upper slots. Thus, we treat $S$ as a ``context" (``hint") to $\A_i$. We generalize from a given hint $S\subset X$ to the ``similar" ones, using the following notion of distance on subsets of $X$:
\begin{align}\label{eq:context-distance}
\DC(S,S')
	:= \inf \textstyle{\sum_{i=1}^n} \D(x_i,x'_i),
\end{align}
where the infimum is taken over all $n\in\N$ and over all $n$-element sequences $\{x_i\}$ and $\{x'_i\}$ that enumerate, possibly with repetitions, all values in $S$ and $S'$, respectively. We justify using $\DC$ via the following lemma:
} 

%% file: sec-appendix.tex
\section{Instance-dependent regret bounds from prior work}
\label{app:instance-dependent}

In this section we discuss instance-dependent regret bounds from prior work on \UCB-style algorithms for the single-slot setting. The purpose is to put forward a concrete mathematical evidence which suggests that $\rankedGridUCB$, $\zooming$ and $\MCzooming$ are likely to satisfy strong upper bounds on regret in the $k$-slot setting (perhaps under some additional assumptions), even if such bounds are beyond the reach of our current techniques. Similarly, we believe that the regret bound for $\rankedContextualZooming$ that we have been able to prove (Theorem~\ref{thm:regret-RankContZoom}) is overly pessimistic. A secondary purpose is to provide more intuition for when these algorithms are likely to excel.

Our story begins with the comparison between the guarantees for $\EXP$ and $\UCB$ in the standard (single-slot, metric-free) bandit setting, and then progresses to Lipschtz MAB and contextual Lispchitz MAB.

In what follows, we let $\mu$ denote the vector of expected rewards in the stochastic reward setting, so that $\mu(x)$ is the expected reward of arm $x$. Let 	
    $\Delta(x) \triangleq \max \mu(\cdot) - \mu(x)$
denote the ``badness'' of arm $x$ compared to the optimum.

\subsection{Standard bandits: $\UCB$ vs. $\EXP$}

Algorithm \EXP~\citep{bandits-exp3} achieves regret
	$R(T) = \tilde{O}(\sqrt{nT})$
against an oblivious adversary. In the stochastic setting, \UCB~\citep{bandits-ucb1} performs much better, with \emph{logarithmic} regret for every fixed $\mu$. More specifically, each arm $x\in X$ contributes only
	$O(\log T)/\Delta(x)$
to regret.
Noting that the total regret from playing arms with $\Delta(\cdot)\leq \delta$ can be a priori upper-bounded by $\delta T$, we bound regret of \UCB{} as:
\begin{align}\label{eq:regret-UCB1}
R(T) = \min_{\delta>0}\left(\delta T+
		\textstyle{\sum_{x\in X:\, \Delta(x)>\delta}}\, \tfrac{O(\log T)}{\Delta(x)}
	\right).
\end{align}
Note that~\eqref{eq:regret-UCB1} depends on $\mu$. In particular, if $\Delta(\cdot) \geq \delta$ then $R(T) = O(\tfrac{n}{\delta}\, \log T)$.

However, for any given $T$ there exists a ``worst-case" pointwise mean $\mu_T$ such that
	$R(T) = \tilde{\Theta}(\sqrt{nT})$
in~\eqref{eq:regret-UCB1}, matching \EXP. The above regret guarantees for \EXP{} and \UCB{} are optimal up to constant factors~\citep{bandits-exp3,sleeping-colt08}.


\subsection{Bandits in metric spaces}

Let $(X,\D)$ denote the metric space. Recall that the \emph{covering number} $N_r(X)$ is the minimal number of balls of radius $r$ sufficient to cover $X$, and the \emph{covering dimension} is defined as
\begin{align}\label{eq:dim-defn-app}
\CovDim(X,\D) \triangleq \inf\{d\geq 0: N_r(X) \leq \alpha\, r^{-d} \quad\forall r>0 \}.
\end{align}
(Here $\alpha>0$ is a constant which we will keep implicit in the notation.)

Against an oblivious adversary, \gridEXP{} has regret
\begin{align}\label{eq:regret-dim-app}
R(T) = \tilde{O}(\alpha\,T^{(d+1)/(d+2)}),
\end{align}
where $d$ is the covering dimension of $(X,\D)$.

For the stochastic setting, \gridUCB{} and the zooming algorithm have better $\mu$-specific regret guarantees in terms of the covering numbers. These guarantees are similar to~\eqref{eq:regret-UCB1} for \UCB. In fact, it is possible, and instructive, to state the guarantees for all three algorithms in a common form.

Consider reward scales
	$\mathcal{S} = \{2^i:\, i\in \N\}$,
and for each scale $r\in \mathcal{S}$ define
	$$X_r = \{ x\in X:\, r<\Delta(x) \leq 2r \}.$$
Then regret~\refeq{eq:regret-UCB1} of \UCB{} can be restated as
\begin{align}\label{eq:regret-Nr}
R(T) = \min_{\delta>0}\left(\delta T+
		\textstyle{\sum_{r\in \mathcal{S}:\, r\geq \delta}}\,
			N_{(\delta,r)}\, \tfrac{O(\log T)}{r}
	\right),
\end{align}
where
    $N_{(\delta,r)} = |X_r|$.
Further, it follows from the analysis in~\citep{Bobby-nips04,LipschitzMAB-stoc08} that regret of \gridUCB{} is \eqref{eq:regret-Nr} with
	$N_{(\delta,r)} = N_\delta(X_r)$.
For the zooming algorithm, the $\mu$-specific bound can be improved to~\eqref{eq:regret-Nr} with
	$N_{(\delta,r)} = N_r(X_r)$.
These results are summarized in Table~\ref{tab:mu-specific}.

\begin{table}[t]
\caption{Regret bounds in terms of covering numbers}
\label{tab:mu-specific}
\vspace{3mm}

\begin{center}
\begin{tabular}{c|l}
algorithm               & regret is \refeq{eq:regret-Nr} with ... \\ \hline
$\UCB$                  & $N_{(\delta,r)} = |X_r|$          \\
$\gridUCB$              & $N_{(\delta,r)} = N_\delta(X_r)$  \\
zooming algorithm       &$N_{(\delta,r)} = N_r(X_r)$          \\
$\contextualZooming$   & $N_{(\delta,r)} = N_r(\XpairsR)$.
\end{tabular}
\end{center}

\end{table}

For the worst-case $\mu$ one could have
	$N_\delta(X_r) = N_\delta(X)$,
in which case the $\mu$-specific bound for $\gridUCB$ essentially reduces to~\eqref{eq:regret-dim-app}.

For the zooming algorithm, the $\mu$-specific bound above implies an improved version of~\eqref{eq:regret-dim-app} with a different, smaller $d$ called the \emph{zooming dimension}:
    $$\ZoomDim(X,\D,\mu) \triangleq \inf\{d\geq 0: N_r(X_r) \leq c\, r^{-d} \quad\forall r>0 \}. $$
Note that the zooming dimension depends on the triple $(X,\D,\mu)$ rather than on the metric space alone.
It can be as high as the covering dimension for the worst-case $\mu$, but can be much smaller (e.g., $d=0$) for ``nice'' problem instances,  see~\citep{LipschitzMAB-stoc08} for further discussion. For a simple example, suppose an \eps-exponential tree metric has a ``high-reward" branch and a ``low-reward" branch with respective branching factors $b\ll b'$. Then the zooming dimension is $\log_{1/\eps}(b)$, whereas the covering dimension is $\log_{1/\eps}(b')$.

\subsection{Contextual bandits in metric spaces}

Let $\mu(x|h)$ denote the expected reward from arm $x$ given context $h$. Recall that the algorithm is given metrics $\D$ and $\DC$ on documents and contexts, respectively, such that for any two documents $x,x'$ and any two contexts $h,h'$ we have
\begin{align*}
	|\mu(x|h) - \mu(x'|h')| \leq \D(x,x') +\DC(h,h').
\end{align*}

\noindent Let $\XC$ be the set of contexts, and $\Xpairs = X\times \XC$ be the set of all (document, context) pairs. More abstractly, one considers the metric space $(\Xpairs,\Dpairs)$, henceforth the \emph{DC-space}, where the metric is
\begin{align*}
	\Dpairs((x,h),\, (x',h')) = \D(x,x')+\DC(h,h').
\end{align*}

We partition $\Xpairs$ according to reward scales
	$r\in \mathcal{S}$:
\begin{align*}
	\Delta(x|h) &\triangleq \max \mu(\cdot|h) - \mu(x|h), \quad x\in X, h\in \XC. \\
	\XpairsR &\triangleq \{ (x,h)\in \Xpairs:\, r<\Delta(x|h) \leq 2r \}.
\end{align*}
Then contextual regret of~\contextualZooming{} can be bounded by~\eqref{eq:regret-Nr} with
	$N_{(\delta,r)} = N_r(\XpairsR)$,
where $N_r(\cdot)$ now refers to the covering numbers in the DC-space (see Table~\ref{tab:mu-specific}).

Further, one can define the \emph{contextual} zooming dimension as
    $$\DIMpairs(X,\D,\mu) \triangleq \inf\{d\geq 0: N_r(X_r) \leq c\, r^{-d} \quad\forall r>0 \}. $$
Then one obtains~\eqref{eq:regret-dim-app} with  $d = \DIMpairs$. In the worst case, we could have $\mu$ such that
	$N_r(\XpairsR) = N_r(\Xpairs)$,
in which case
    $\DIMpairs \leq \CovDim(\Xpairs,\Dpairs)$.

The regret bounds for~\contextualZooming{} can be improved by taking into account ``benign" context arrivals: effectively, one can prune the regions of $\XC$ that correspond to infrequent context arrivals, see~\citep{contextualMAB-slivkins09} for details. This improvement can be especially significant if
	$\CovDim(\XC,\DC) > \CovDim(X,\D)$.